\documentclass{pepsart}

\usepackage[utf8]{inputenc} 
\usepackage{wasysym}
\usepackage{chemfig}
\usepackage{lineno}

\usepackage{makecell}
\usepackage{changes}
\usepackage{tabularx}

\usepackage{multirow}
\usepackage[utf8]{inputenc}
\usepackage[T1]{fontenc}
\usepackage{hyperref}
\usepackage{url}
\usepackage{booktabs}
\usepackage{amsfonts}
\usepackage{nicefrac}
\usepackage{microtype}
\usepackage{amsmath,amsthm,amssymb}
\hypersetup{
    colorlinks=true,
    citecolor=blue,
    filecolor=black,
    linkcolor=blue,
    urlcolor=black
}

\usepackage{array}

\usepackage{subcaption}

\usepackage{comment,color}

\usepackage[section]{placeins}

\usepackage{bm}
\usepackage{wrapfig}
\usepackage{graphicx}

\usepackage{listings}

\usepackage{color}
\definecolor{codegreen}{rgb}{0,0.6,0}
\definecolor{codegray}{rgb}{0.5,0.5,0.5}
\definecolor{codepurple}{rgb}{0.58,0,0.82}
\definecolor{backcolour}{rgb}{0.95,0.95,0.92}

\lstdefinestyle{mystyle}{
    backgroundcolor=\color{backcolour},
    commentstyle=\color{codegreen},
    keywordstyle=\color{magenta},
    numberstyle=\tiny\color{codegray},
    stringstyle=\color{codepurple},
    basicstyle=\ttfamily\footnotesize,
    breakatwhitespace=false,
    breaklines=true,
    captionpos=b,
    keepspaces=true,
    numbers=left,
    numbersep=5pt,
    showspaces=false,
    showstringspaces=false,
    showtabs=false,
    tabsize=2
}

\newtheorem{proposition}{Proposition}

\def\vx{{\bm{x}}}
\def\vy{{\bm{y}}}
\def\vtheta{{\bm{\theta}}}

\begin{document}

\begin{frontmatter}

\begin{fmbox}
\dochead{Research}


\title{Machine-Learning Emulation of Satellite Greenhouse Gas Retrievals: Stability over Time}


\author[
   addressref={aff1},
   corref={aff1},
   email={nugzar.gognadze@eurecom.fr}
]{\inits{N}\fnm{Nugzar} \snm{Gognadze}}
\author[
   addressref={aff1},
   email={motonobu.kanagawa@eurecom.fr}
]{\inits{M}\fnm{Motonobu} \snm{Kanagawa}}
\author[
   addressref={aff2},
   email={someya.yu@nies.go.jp}
]{\inits{Y}\fnm{Yu} \snm{Someya}}
\author[
   addressref={aff2},
   email={yashiro.hisashi@nies.go.jp}
]{\inits{H}\fnm{Hisashi} \snm{Yashiro}}


\address[id=aff1]{%
  \orgname{Data Science Department, EURECOM},
  \street{450 Route des Chappes},
  \city{Biot},
  \postcode{06410},
  \cny{France}
}

\address[id=aff2]{%
  \orgname{National Institute for Environmental Studies},
  \street{16-2 Onogawa},
  \city{Tsukuba, Ibaraki},
  \postcode{305-8506},
  \cny{Japan}
}


\begin{artnotes}

\end{artnotes}

\end{fmbox}


\begin{abstractbox}

\begin{abstract} 
Retrieval algorithms are used to estimate atmospheric concentrations of greenhouse gases (GHGs), such as carbon dioxide (CO$_2$) and methane (CH$_4$), by solving inverse problems from high-spectral-resolution satellite radiance measurements. However, these algorithms are computationally expensive, which makes real-time estimation at scale difficult.
Machine-learning models have therefore been proposed as fast emulators of retrieval algorithms. Most existing studies, however, evaluate them only on test data from the same period as the training data.

We study the stability over time of such emulators using data from the Greenhouse Gases Observing SATellite (GOSAT). We show that prediction accuracy generally deteriorates when the test period moves away from the training period. We also show that including time as an input feature substantially improves XCH$_4$ prediction for Lasso and neural-network models. Among the methods considered, a simple Lasso model performs as well as or better than more complex methods such as neural networks, and yields more stable predictions over time. We further validate the results using the Total Carbon Column Observing Network (TCCON), a ground-based observation network. On the TCCON-matched dataset, the time-augmented Lasso achieves errors against TCCON that are comparable to the disagreement between GOSAT and TCCON for both XCO$_2$ and XCH$_4$.
\end{abstract}




\end{abstractbox}
%

\end{frontmatter}





\section{Introduction}

Monitoring atmospheric greenhouse-gas (GHG) concentrations, in particular carbon dioxide (CO$_2$) and methane (CH$_4$), is essential for studying climate change and the global carbon cycle.
Satellite missions such as the Greenhouse Gases Observing SATellite (GOSAT) enable large-scale monitoring of these gases through global observations \cite{yokota2009global}. GOSAT carries the Thermal and Near-infrared Sensor for Carbon Observation--Fourier Transform Spectrometer (TANSO-FTS) \cite{kuze2009thermal}, which measures short-wavelength infrared (SWIR) radiation reflected from the Earth's surface at high spectral resolution. These spectra contain absorption features of CO$_2$ and CH$_4$, from which the column-averaged dry-air mole fractions XCO$_2$ and XCH$_4$ are retrieved.

Traditional retrieval methods for such measurements solve inverse problems using computationally intensive full-physics models that account for changes in light paths caused by atmospheric particles \cite{yoshida2011retrieval,yoshida2013improvement,butz2011toward,acos2012,acoss2018improved}. These models incorporate atmospheric states and instrument characteristics, and retrieve XCO$_2$ and XCH$_4$ through iterative optimization. However, their computational cost makes them difficult to use in large-scale applications.

To address these limitations, machine-learning approaches have been developed. In these approaches, the model learns a direct mapping from satellite-observed SWIR spectra and other input variables to the retrieved column-averaged concentrations XCO$_2$ and XCH$_4$. The trained model can then be used as a fast emulator of the retrieval algorithm, reducing the need for repeated radiative-transfer calculations at prediction time.

\subsection{Related Work} \label{sec:related-work}

Several studies have demonstrated the potential of machine learning for greenhouse-gas retrieval, but most evaluations have used training and test data drawn from the same overall period.

In David et al.~\cite{david2021xco}, a neural network is used to estimate XCO$_2$ and surface pressure from Orbiting Carbon Observatory-2 (OCO-2) radiance measurements over land in nadir mode, together with observation-geometry variables such as solar zenith angle and relative azimuth.
The model is trained on about 131000 observations from the even-numbered months between January~2015 and August~2018, and evaluated on observations from the odd-numbered months within the same period. Thus, the evaluation tests performance on held-out months within the same overall period, but does not directly assess performance on observations from later years.

Breon et al.~\cite{breon2022neural} study a neural-network approach for estimating XCO$_2$ from OCO-2 spectra. The analysis uses OCO-2 observations from February~2015 to December~2019. A 3\% random sample of the full dataset is used for training, and the sampled observations are excluded from the subsequent evaluation. Thus, the evaluation uses a random split within the same overall period, not a later-year test set.

Zhao et al.~\cite{zhao2022atmospheric} proposed a two-step machine-learning method for GOSAT-based XCO$_2$ retrieval. Their proof-of-concept study used 3300 synthetic samples generated from 33 base pressure, temperature, and CO$_2$ profiles from Australia on 26--28 July 2009, with a random 90/10 train--test split. Since both training and test samples were generated from the same limited set of base profiles, the study did not evaluate the method on independent observations from later dates.

Gong et al.~\cite{gong2024estimation} studied XCO$_2$ estimation from CrIS thermal-infrared satellite data using three ensemble-learning methods: Extreme Gradient Boosting (XGBoost), Extremely Randomized Trees (ERT), and Gradient Boosted Regression Trees (GBRT). Their dataset combined CrIS measurements with meteorological, surface, vegetation, elevation, and observation-geometry variables, with TCCON observations used as reference data. The dataset was randomly split into training (80\%), testing (10\%), and validation (10\%) sets. Thus, the evaluation was based on random splits from the same 2019 dataset, rather than on a later-period test set.

Cui et al.~\cite{cui2024estimating} proposed SpatialFusionNet, a spatial-feature fusion module for XCO$_2$ estimation using OCO-2 data over China. Their main OCO-2 evaluation used 19183 spatially matched samples from 2015 in sample-based cross-validation. They also compared model predictions with TCCON observations at the Hefei site from 2015 to 2017. This provides an external site-level check, but the reported TCCON metrics are aggregated over the full comparison period rather than reported separately by year. Thus, the primary OCO-2 evaluation was based on cross-validation of matched samples, not on a separate later-period test set.

Li et al.~\cite{Li2024UNMAMO} proposed a hybrid machine-learning approach for XCH$_4$
estimation using TROPOMI and GOSAT data. They used radiative-transfer
simulations to select methane-sensitive TROPOMI bands and trained random-forest
models using GOSAT XCH$_4$ retrievals as learning targets. Their evaluation
included cross-validation, comparison with TCCON observations in 2021, and
site-level temporal comparisons. However, they did not systematically evaluate
out-of-time prediction by training the emulator on an earlier period and testing
it on later satellite observations. They also noted that the spatial and temporal
representativeness of machine-learned retrieval models remains an important
issue.

Reuter et al.~\cite{reuter2025retrieving} explicitly addressed the difficulty of applying machine-learning retrievals to future data, noting that present-day spectra may not represent future conditions well because greenhouse-gas concentrations increase over time. To mitigate this, they proposed a hybrid approach in which spectra are modified so that the training data cover a wider range of CO$_2$ and CH$_4$ concentrations. Their evaluation, however, was based on an Observing System Simulation Experiment (OSSE), rather than real satellite observations. In this setting, the true atmospheric concentrations are known by construction, and the study assumes no systematic error in the training truth. This provides a controlled setting for studying temporal shift, but differs from the empirical setting considered here, where the emulator is trained on actual retrieval outputs and tested on later observations.

\subsection{Contributions}

The operationally relevant setting is one in which machine-learning models are trained on past data and applied to later observations. For example, training on 2020 data and testing on unseen 2021 data is more informative than using random train--test splits within 2020. As discussed by Efron et al.~\cite[Section~6]{efron2020prediction}, random train--test splits can give overly optimistic results, because they evaluate the model on test data drawn under conditions similar to those of the training data, whereas later observations may differ because of temporal drift. For machine-learning emulators of satellite greenhouse-gas retrievals, accuracy on later observations is therefore practically important but has been insufficiently studied.

This paper makes three contributions.

First, using GOSAT data, we evaluate four machine-learning models---Lasso, a neural network, k-nearest neighbours regression, and XGBoost---as emulators of satellite greenhouse-gas retrievals in the operationally relevant setting where models are trained on past data and evaluated on later observations. We show that prediction accuracy generally deteriorates as the test period moves away from the training period.

Second, we show that a simple modification---adding time itself as an input feature---substantially improves out-of-time prediction for XCH$_4$ for Lasso and neural-network models. This appears to work by compensating for a temporal mismatch in the learned input--output relation. By contrast, the effect is small for XCO$_2$ and mixed for k-nearest neighbours and XGBoost.

Third, we show that a simple time-augmented Lasso performs as well as or better than more complex methods and gives more stable predictions over time. External validation against the Total Carbon Column Observing Network (TCCON) \cite{wunch2011total,laughner2024ggg2020} shows that the Lasso emulator has errors of the same order as the GOSAT--TCCON discrepancy. This suggests that it can serve as a practical surrogate for the retrieval algorithm. The fitted coefficients are also physically interpretable: the sparsity of Lasso makes clear which inputs drive the predictions, and these are concentrated in known absorption bands and related auxiliary variables.

The rest of the paper is organized as follows.
Section~\ref{sec:emulating-retrieval} describes the machine-learning emulation
of retrieval algorithms and discusses the role of time in out-of-time prediction.
Section~\ref{sec:data-source} presents the data sources and preprocessing
steps. Section~\ref{sec:OOT-pred-performance} compares the out-of-time prediction
performance of the learning methods. Section~\ref{sec:detailed-analysis-Lasso-emulator}
gives a detailed analysis of the Lasso emulator, including temporal and spatial
diagnostics, coefficient analysis, and validation against TCCON.

\section{Emulating Retrieval with Machine Learning}
\label{sec:emulating-retrieval}

This section describes how a machine-learning model can be used to emulate a retrieval algorithm, and explains why accurate out-of-time prediction is difficult.

\subsection{Retrieval Algorithm}
\label{sec:retrieval-algo}

A retrieval algorithm maps an input vector $\vx$ to an output vector $\vy$:
\[
\vy = f^*(\vx).
\]
The input $\vx$ contains the observed radiance spectra, a priori estimates of the target atmospheric states (e.g., XCO$_2$ and XCH$_4$), and other environmental variables. The output $\vy$ contains the retrieved atmospheric states.

For example, in the NIES retrieval algorithm used in our experiments \cite{yoshida2011retrieval,yoshida2013improvement,someya2023update}, the retrieved state vector $\vy$\footnote{The state vector contains CO$_2$ and CH$_4$ in 15 vertical layers, their column averages XCO$_2$ and XCH$_4$, and other parameters such as aerosol and cloud parameters, surface pressure, temperature shift, and surface albedo.} is defined as the maximizer of the posterior density given the observed spectra \cite{rodgers2000inverse}. The posterior combines a prior distribution for the state vector with a likelihood for the observed spectra. The a priori distribution is specified from simulated a priori estimates, while the likelihood is defined through a radiative-transfer forward model with Gaussian observation error that accounts for multiple scattering by clouds and aerosols. Thus, the input vector $\vx$ consists of the observed spectra, the a priori quantities entering the prior distribution, and other auxiliary variables entering the likelihood; see Table~\ref{tab:feature_descriptions} in Section~\ref{sec:data-source} for details.

Computing $f^*(\vx)$ requires numerical optimization of the state vector, with repeated radiative-transfer calculations, and is therefore computationally expensive. Running the retrieval at scale in near real time may thus be infeasible.

However, it is feasible to apply the retrieval offline to many historical inputs $\vx_1,\dots,\vx_n$ and obtain the corresponding outputs $\vy_1,\dots,\vy_n$, yielding a dataset
\begin{equation}
\label{eq:training-data}
(\vx_1,\vy_1),\dots,(\vx_n,\vy_n),
\qquad
\text{where}
\qquad
\vy_i = f^*(\vx_i), \quad i=1,\dots,n.
\end{equation}
We use these input--output pairs to construct a computationally cheaper emulator of the retrieval algorithm.

\subsection{Emulation}

The aim of emulation is to approximate the computationally expensive retrieval function $f^*$ by a much cheaper function $f_{\vtheta}$:
\[
f^*(\vx) \approx f_{\vtheta}(\vx).
\]
Here $f_{\vtheta}$ denotes a surrogate model with parameter vector $\vtheta$. The parameters are estimated from the training data so that the surrogate predicts the retrieval outputs accurately for unseen inputs.

For simplicity, we now restrict attention to a scalar output $y \in \mathbb{R}$, such as XCO$_2$ or XCH$_4$. We write the input as a $d$-dimensional vector $\vx=(x_1,\dots,x_d)\in\mathbb{R}^d$. Then the training data in \eqref{eq:training-data} becomes
\begin{equation}
\label{eq:training-data-scalar}
(\vx_1,y_1),\dots,(\vx_n,y_n)\in \mathbb{R}^d \times \mathbb{R},
\qquad
\text{where}
\qquad
y_i=f^*(\vx_i), \quad i=1,\dots,n.
\end{equation}

A natural criterion for fitting the emulator $f_{\vtheta}$ is the mean squared error
\begin{equation}
\label{eq:MSE}
\frac{1}{n}\sum_{i=1}^n \bigl(y_i-f_{\vtheta}(\vx_i)\bigr)^2.
\end{equation}
The parameter vector $\vtheta$ is estimated by minimizing this criterion, possibly with a regularization term to control model complexity and reduce overfitting.

\label{sec:lasso-expl}

One simple example is the Lasso \cite{tibshirani1996regression,hastie2015statistical}, that is, $\ell_1$-regularized linear regression. It models the output as
\begin{equation}
\label{eq:Lasso-without-time}
f_{\vtheta}(\vx)=\theta_0+\sum_{j=1}^d \theta_j x_j,
\end{equation}
where $\vtheta=(\theta_0,\theta_1,\dots,\theta_d)$ is the parameter vector. The parameters are estimated by minimizing
\begin{equation*}
\min_{\vtheta}\;
\frac{1}{n}\sum_{i=1}^n \bigl(y_i-f_{\vtheta}(\vx_i)\bigr)^2
+\lambda \sum_{j=1}^d |\theta_j|,
\end{equation*}
where $\lambda\ge 0$ is a regularization parameter.\footnote{
The intercept $\theta_0$ is not penalized; the $\ell_1$ penalty is applied only to the feature coefficients.
} The $\ell_1$ penalty encourages sparsity, so that only a relatively small number of features are assigned nonzero coefficients.

We also consider neural networks, k-nearest neighbours regression, and XGBoost; see Section~\ref{sec:models-evaluation-protocol}.

\subsection{Why Out-of-Time Prediction Is Difficult}

The emulator should predict retrieval outputs accurately not only within the
training period, but also in later periods. Good in-period accuracy is not
sufficient. Even if the retrieval algorithm itself is fixed, the distribution
of the inputs may change over time. A model trained under one input distribution
may then give a poor approximation under another, especially when the model
class is restricted.

This difficulty is illustrated by
Figures~\ref{fig:Lasso_timeline_xco2_land_ocean}
and~\ref{fig:Lasso_timeline_xch4_land_ocean} in
Section~\ref{sec:Lasso-temporal-behaviour}. These figures show representative
Lasso fits for XCO$_2$ and XCH$_4$ over land and ocean. The models were trained
on 2020 data and evaluated on held-out 2020 data and on out-of-time data from
2021 to 2023. For XCO$_2$, the deterioration over time is mild. For XCH$_4$,
by contrast, prediction accuracy deteriorates rapidly after 2020. Similar
patterns are also seen for the other models; see
Section~\ref{sec:OOT-pred-performance}.

A natural explanation is covariate shift
\cite{shimodaira2000improving,sugiyama2012machine}. In our setting, the input
vector includes observed radiance spectra and auxiliary retrieval variables.
As atmospheric CO$_2$ and CH$_4$ concentrations increase, the spectra observed
in later years need not have the same distribution as those observed in the
training period. A related concern was emphasized by Reuter et al.~\cite{reuter2025retrieving}, who noted that present-day spectra may not
represent future conditions well when greenhouse-gas concentrations increase
over time.

The XCH$_4$ case shows the problem more concretely; see
Appendix~\ref{sec:learning-theory} for a population-level formulation. During
the 2020 training period, the difference between retrieved XCH$_4$ and the
a priori XCH$_4$ is approximately constant. This suggests the local approximation
\begin{equation}
\label{eq:XCH4-hypothesis-overfit}
f^*(\vx) \approx x_{\rm XCH4}^{\rm ap} + r(\vx),
\qquad
r(\vx) \approx \text{constant},
\end{equation}
where $x_{\rm XCH4}^{\rm ap}$ denotes the a priori XCH$_4$, one component of
the input vector $\vx$, and $r(\vx)$ denotes the residual difference between
retrieved and a priori XCH$_4$. Thus, on the 2020 training distribution, the retrieval output can be approximated
well by the a priori XCH$_4$ plus an approximately constant correction.

This relation does not remain valid in later years. From 2021 onward, the
difference between retrieved and a priori XCH$_4$ increases over time and is no
longer approximately constant. Thus, retrieved XCH$_4$ is not well represented
by a priori XCH$_4$ plus a fixed correction. It also depends on the spectra and
other inputs in a way that is not identified from the 2020 training data alone.
The resulting failure is therefore not ordinary overfitting within 2020, but a
failure to extrapolate the learned input--output relation to later input
distributions.

The same point can be stated in terms of the training criterion. The empirical
loss in \eqref{eq:MSE} estimates prediction error under the training-period
input distribution:
\begin{align*}
\frac{1}{n}\sum_{i=1}^n
\bigl(f^*(\vx_i)-f_{\vtheta}(\vx_i)\bigr)^2
&\approx
\int
\bigl(f^*(\vx)-f_{\vtheta}(\vx)\bigr)^2
p_{\rm tr}(\vx)\,d\vx .
\end{align*}
The relevant out-of-time error is instead
\begin{align*}
\int
\bigl(f^*(\vx)-f_{\vtheta}(\vx)\bigr)^2
p_{\rm te}(\vx)\,d\vx ,
\end{align*}
where $p_{\rm tr}$ and $p_{\rm te}$ denote the input distributions in the
training and later test periods. When $p_{\rm tr}\neq p_{\rm te}$, small
training-period error does not imply small out-of-time error. This provides one
explanation for the rapid deterioration observed for XCH$_4$.

A standard response to covariate shift is importance weighting, which reweights
the training loss by an estimate of the density ratio
$p_{\rm te}(\vx)/p_{\rm tr}(\vx)$
\cite{sugiyama2012density}. We do not pursue this approach here. Reliable
density-ratio estimation is difficult in this problem because the input vector
is high-dimensional, and the future test distribution is not available at
training time.
\subsection{Time as an Input Feature}

The discussion above suggests that part of the out-of-time deterioration is caused by time-dependent changes that are not fully accounted for by the original emulator inputs.

Each input vector $\vx$ is associated with an observation time $t$, so we may write the retrieval output as $f^*(\vx(t))$. Time is not itself an input to the retrieval algorithm; it enters only indirectly through the input vectors.

We augment the emulator input by adding time $t$ as an additional feature.
The emulator is then trained to minimize
\begin{equation}
\frac{1}{n} \sum_{i=1}^n \left( f^*(\vx(t_i)) - f_{\vtheta}(\vx(t_i), t_i) \right)^2,
\end{equation}
where $t_i$ is the observation time of the $i$th training sample.

For example, the time-augmented Lasso model takes the form\footnote{Time is normalized so that the beginning and the end of the training period correspond to 0 and 1, respectively; see Section~\ref{sec:data-source}.}
\begin{equation}
\label{eq:Lasso-with-time}
f_{\vtheta}(\vx,t)=\theta_0+\sum_{j=1}^d \theta_j x_j+\theta_{d+1} t,
\end{equation}
where $\theta_{d+1}\in\mathbb{R}$ is the coefficient of the time feature. More elaborate time features, such as polynomial or trigonometric functions of time, could also be used, but we restrict attention to a linear time term for simplicity. The corresponding Lasso objective is
\begin{equation*}
\min_{\vtheta}\;
\frac{1}{n}\sum_{i=1}^n \bigl(y_i-f_{\vtheta}(\vx_i,t_i)\bigr)^2
+\lambda \sum_{j=1}^{d+1} |\theta_j|.
\end{equation*}
The other machine-learning models considered in this paper can also be augmented in the same way.

Figures~\ref{fig:Lasso_timeline_xco2_land_ocean} and \ref{fig:Lasso_timeline_xch4_land_ocean} in
Section~\ref{sec:Lasso-temporal-behaviour} compare representative Lasso fits with and without the time feature for XCO$_2$ and XCH$_4$, respectively. For XCH$_4$, the time-augmented Lasso is clearly more accurate than the original Lasso. Over the 2021--2023 test period, its predictions remain much closer to the retrieved XCH$_4$, whereas the original Lasso increasingly underestimates it. The figure also shows the fitted linear time component of the time-augmented Lasso. This component closely matches the upward trend in XCH$_4$, suggesting that the model has learned the time trend from the 2020 training data. The time feature thus compensates for the missing temporal trend in the a priori XCH$_4$.

For XCO$_2$, the effect of the time feature differs between land and ocean. Over land, the time-augmented Lasso is comparable to the original Lasso on the 2021 test data, but gradually overestimates the retrieved XCO$_2$ in 2022--2023 and becomes slightly less accurate. A plausible explanation is that the a priori XCO$_2$ already contains much of the relevant temporal trend, so that the additional linear time term mildly over-corrects in later years. Over ocean, by contrast, the fitted time term appears to play little role, and the time-augmented and original Lasso models behave almost identically. Thus, the usefulness of time augmentation depends on the residual structure left by the original inputs. Appendix~\ref{sec:learning-theory} gives a population-level formulation of this point.

\section{Data and Preprocessing}
\label{sec:data-source}

This section describes the data and preprocessing used in the experiments. We
specify the GOSAT inputs and retrieval outputs, the TCCON collocation procedure
for external validation, and the construction of the training and test datasets.
All standardization parameters are computed from the 2020 training data only,
so that no information from the test periods enters model fitting.

\paragraph{Input variables.}
The retrieval algorithm to be emulated is the NIES retrieval algorithm mentioned in Section~\ref{sec:retrieval-algo} \cite{yoshida2011retrieval,yoshida2013improvement,someya2023update}. Its input vector~$\vx$ consists of observed radiance spectra, gas-specific a priori variables, and other auxiliary variables, as summarized in \autoref{tab:feature_descriptions}.
\begin{itemize}
    \item The observed spectra are from the GOSAT TANSO-FTS SWIR Level 1B product V230.231. They consist of radiance measurements at 4261 discrete spectrometer channels. Each channel corresponds to a central wavenumber (cm$^{-1}$) and a narrow spectral interval. For example, \texttt{ch3760} corresponds to 4799.61~cm$^{-1}$ in the CO$_2$ absorption band.

    \item The a priori variables are derived from a global atmospheric transport model \cite{maksyutov2008nies,saeki2013global}. For XCO$_2$, we use the CO$_2$ a priori profile \texttt{co2\_1}, \ldots, \texttt{co2\_15} and the column-mean a priori value \texttt{xco2\_ap}. For XCH$_4$, we use the CH$_4$ a priori profile \texttt{ch4\_1}, \ldots, \texttt{ch4\_15} and the column-mean a priori value \texttt{xch4\_ap}.

    \item The remaining auxiliary variables include Doppler variables, location and height, observation geometry, surface pressure, and temperature profile variables. Each angular variable $\varphi$ is represented by its sine and cosine, $\sin(\pi \varphi / 180)$ and $\cos(\pi \varphi / 180)$. For example, \texttt{SZs} denotes the sine of the solar zenith angle, and \texttt{SatAs} denotes the sine of the satellite azimuth angle. The temperature variables are indexed by vertical level; for example, \texttt{T15} denotes the temperature at level 15.
\end{itemize}
Thus, for each target variable, each observation is represented by 4323 input variables without time and 4324 input variables with time.

\paragraph{Output variables.}
The output variables are the retrieved XCO$_2$ and XCH$_4$ from the GOSAT
TANSO-FTS SWIR Level 2 (L2) V03.00 product \cite{someya2023update}. We train
separate models for each target variable and surface type. The XCO$_2$ and
XCH$_4$ models use their corresponding a priori variables, and within each target
variable separate models are fitted for land and ocean observations.

\begin{table}[t]
\centering
\caption{Input variables used by the retrieval emulator.
Time $t$ is used only in the time-augmented learning models and is not an input
to the retrieval algorithm itself.}
\label{tab:feature_descriptions}

\small
\renewcommand{\arraystretch}{1.15}
\setlength{\tabcolsep}{3pt}

\begin{tabularx}{\textwidth}{@{}
>{\raggedright\arraybackslash}p{2.5cm}
>{\raggedright\arraybackslash}p{3.5cm}
>{\centering\arraybackslash}p{0.8cm}
>{\raggedright\arraybackslash}p{1.8cm}
>{\raggedright\arraybackslash}X
@{}}
\toprule
\textbf{Category} & \textbf{Variable names} & \textbf{Dim.} & \textbf{Unit} & \textbf{Description} \\
\midrule
Spectra
& \texttt{ch0}, \ldots, \texttt{ch4260}
& 4261
& \makecell[l]{W\,cm$^{-2}$\,sr$^{-1}$\\(cm$^{-1}$)$^{-1}$}
& Observed radiance spectra. \\

A priori CH$_4$
& \texttt{ch4\_1}, \ldots, \texttt{ch4\_15}
& 15
& ppm
& A priori CH$_4$ profile at levels 1--15. \\

A priori CO$_2$
& \texttt{co2\_1}, \ldots, \texttt{co2\_15}
& 15
& ppm
& A priori CO$_2$ profile at levels 1--15. \\

A priori column means
& \texttt{xch4\_ap}, \texttt{xco2\_ap}
& 2
& ppm
& Column-mean a priori values for XCH$_4$ and XCO$_2$. \\

Doppler variables
& \texttt{dop\_v\_earth}, \texttt{dop\_v\_sat}, \texttt{dop\_v\_sun}
& 3
& m\,s$^{-1}$
& Doppler velocities due to Earth rotation, satellite motion, and the Sun. \\

Location and height
& \texttt{h\_avg}, \texttt{lat}, \texttt{lon}
& 3
& --
& Averaged height in the IFOV, latitude, and longitude. \\

Observation geometry
& \texttt{SAs}, \texttt{SAc}, \texttt{SZs}, \texttt{SZc}, \texttt{SatAs}, \texttt{SatAc}, \texttt{SatZs}, \texttt{SatZc}
& 8
& --
& Sine and cosine representations of solar and satellite azimuth and zenith angles. \\

Surface pressure
& \texttt{Sp}
& 1
& hPa
& A priori surface pressure. \\

Temperature profile
& \texttt{T1}, \ldots, \texttt{T31}
& 31
& K
& A priori temperature profile at levels 1--31. \\

Time
& \texttt{t}
& 1
& --
& Normalized observation time for the time-augmented learning models. \\
\bottomrule
\end{tabularx}
\end{table}

\paragraph{External validation dataset.}
We use the Total Carbon Column Observing Network (TCCON) \cite{wunch2011total,laughner2024ggg2020} as an external validation dataset. TCCON provides high-quality ground-based observations of greenhouse-gas concentrations. The TCCON data from the sites listed in \autoref{tab:tccon_sites} in Appendix~\ref{app:additional-tables} were obtained from the TCCON Data Archive hosted at \url{https://tccondata.org}. Each site is cited through its corresponding public TCCON dataset reference, and the corresponding dataset DOIs are given in the reference list.

For each TCCON observation, we identify matching GOSAT observations using the following criteria: a latitude difference of at most 2.0 degrees, a longitude difference of at most 2.0 degrees, a height difference of at most 500~m, and a time difference of at most 30 minutes. This yields 5225 matched GOSAT--TCCON pairs from 2021 to 2023.

\paragraph{Time.}
For the time-augmented models, the observation time $t$ is normalized so that the start and end of 2020 correspond to $0$ and $1$, respectively. Values greater than $1$ therefore indicate observations from 2021 onward.

\paragraph{Missing values.}
Among the final input variables, missing values occurred only in one radiance
channel, \texttt{ch4260}: 245 values out of 138310 observations. These values
were imputed by averaging chronologically adjacent non-missing observations
within the same channel. No neighbouring spectral channels were used.

\paragraph{Standardization.}
For each input variable, we computed the mean and standard deviation from the
corresponding 2020 training subset only. Using these quantities, we standardized
the same variable in the training subset, the corresponding 2020 in-period test
set, the 2021--2023 out-of-time test sets, and the TCCON-matched evaluation
dataset. No information from the evaluation data was used in this step.

After preprocessing, the 2020 dataset contains 87013 land observations and 38673 ocean observations. These observations are used for training and in-period testing, while the 2021--2023 observations are reserved for out-of-time evaluation.

\section{Out-of-Time Prediction Performance}
\label{sec:OOT-pred-performance}

This section compares the out-of-time prediction performance of four learning
methods: Lasso, neural networks, $k$-nearest neighbours regression, and XGBoost.
For each method, we consider two versions, one without time and one with the
normalized time feature. The aim is to assess how well these methods remain
accurate on later observations, and whether explicit time information improves
out-of-time prediction differently for XCO$_2$ and XCH$_4$.

\subsection{Learning Models and Evaluation Protocol}
\label{sec:models-evaluation-protocol}

We compare four representative machine-learning models: Lasso~\cite{tibshirani1996regression},
neural networks (NN)~\cite{goodfellow2016deep}, $k$-nearest neighbours
($k$-NN) regression~\cite{stone1977consistent}, and
XGBoost~\cite{chen2016xgboost}.
Each method is fitted or tuned using mean-squared prediction error, where applicable.

Lasso is a linear regression model with sparsity-inducing $\ell_1$
regularization; see Section~\ref{sec:lasso-expl} and Appendix~\ref{app:Lasso}.
The neural network is a fully connected feed-forward network with ReLU
activations~\cite{nair2010rectified}, optimized by Adam~\cite{kingma2017}
with early stopping; see Appendix~\ref{app:NN}. The $k$-NN model predicts by
averaging the outputs of the $k$ nearest training inputs under a selected
distance metric; see Appendix~\ref{app:knn}. XGBoost constructs an additive
ensemble of decision trees by gradient boosting; see Appendix~\ref{app:xgb}.

For each learning method, target gas, and surface type, we fit two versions:
one without the time feature and one with the normalized time feature. This
allows a direct comparison of whether adding time improves out-of-time
prediction.

For each run, we randomly select 80\% of the 2020 data for training.
Hyperparameters are selected by cross-validation within this training set,
using a random 3:1 split between sub-training and validation data. The remaining
20\% of the 2020 data is used as an in-period test set. The 2021, 2022, and
2023 datasets are used only as out-of-time test sets.

For each target variable, test set, and trained model, we compute the normalized
root mean squared error
\begin{equation}
\label{eq:NRMSE}
\mathrm{NRMSE}
=
\left\{
\frac{1}{N}
\sum_{i=1}^N
\left(
\frac{y_i - \hat y_i}{y_i}
\right)^2
\right\}^{1/2},
\end{equation}
where \(N\) is the number of observations in the test set, \(y_i\) is the GOSAT
retrieval for the \(i\)th test observation, and \(\hat y_i\) is the corresponding
model prediction. We repeat the experiment 10 times and report the mean and
standard deviation of \(\mathrm{NRMSE}\) across runs.

\subsection{Prediction Results}
\label{sec:prediction-results}

\begin{figure}[t]
    \centering
    \includegraphics[width=0.95\textwidth]{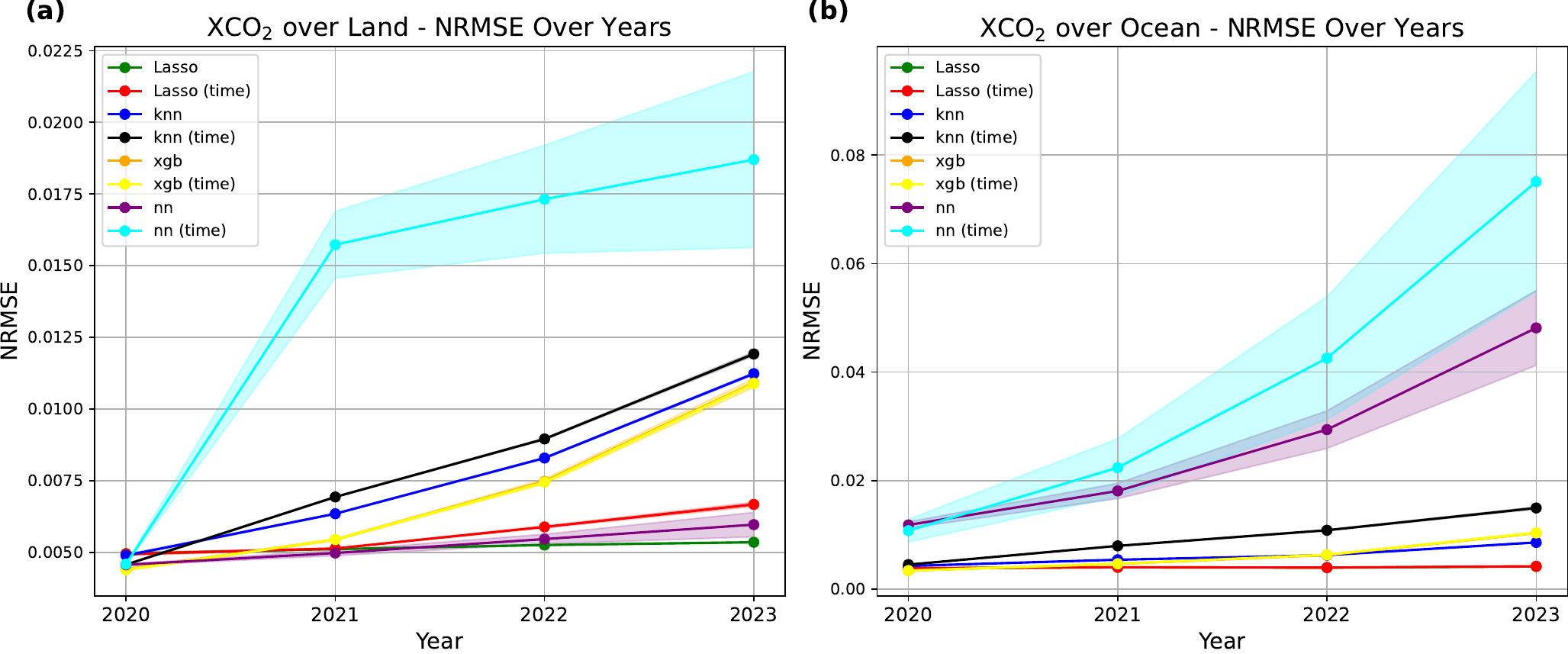}
\caption{Out-of-time NRMSE for XCO$_2$.
NRMSE is shown for the 2020 in-period test set and the 2021--2023 out-of-time
test sets for $k$-NN, XGBoost, Lasso, and NN over land (left) and ocean
(right). Each model is trained 10 times on different random 80/20\% splits of
the corresponding 2020 data. Solid lines show the mean across runs, and shaded
bands indicate one standard deviation across runs.}
    \label{fig:OOT_xco2_land_ocean}
\end{figure}

\begin{figure}[t]
    \centering
    \includegraphics[width=0.95\textwidth]{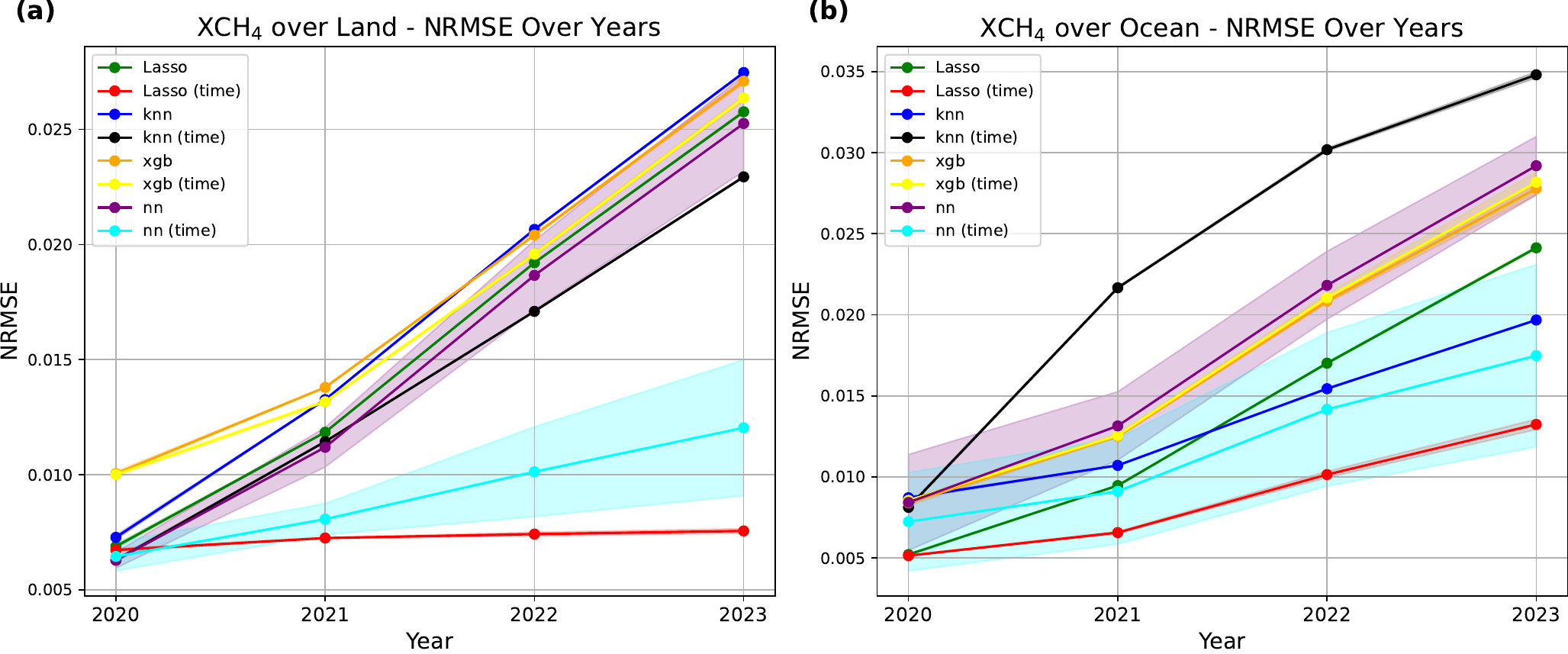}
    \caption{Out-of-time NRMSE for XCH$_4$.
    NRMSE is shown for the 2020 in-period test set and the 2021--2023 out-of-time
    test sets for $k$-NN, XGBoost, Lasso, and NN over land (left) and ocean
    (right). Each model is trained 10 times on different random 80/20\% splits of
    the corresponding 2020 data. Solid lines show the mean across runs, and shaded
    bands indicate one standard deviation across runs.}
    \label{fig:OOT_xch4_land_ocean}
\end{figure}

Figures~\ref{fig:OOT_xco2_land_ocean} and~\ref{fig:OOT_xch4_land_ocean}
show the mean NRMSE and the corresponding $\pm 1$ standard deviation bands
across 10 independent runs for XCO$_2$ and XCH$_4$, respectively. The comparison
covers all four learning methods, with and without the normalized time feature,
on the 2020 in-period test set and the 2021--2023 out-of-time test sets.
Full numerical results are reported in Appendix~\ref{app:additional-tables},
Tables~\ref{tab:OOT_land_xco2}--\ref{tab:OOT_ocean_xch4}. Except for the neural
network, the standard deviations are so small that the shaded bands are barely
visible.

The main result is that Lasso gives the most stable out-of-time performance
overall. Time augmentation is useful mainly for XCH$_4$, but its benefit is
model-dependent. For XCH$_4$, adding time substantially improves later-year
prediction for Lasso and the neural network, with the time-augmented Lasso
giving the lowest and most stable errors overall. For XCO$_2$, the effect is
small: over ocean, the Lasso results are essentially unchanged, while over land
the time-augmented version deteriorates mildly in later years.

\subsubsection{Interpretation of the XCO$_2$ Results}

For XCO$_2$, all methods are accurate on the 2020 in-period test set, but their
out-of-time behaviour differs. The Lasso-based models deteriorate only mildly
from 2021 to 2023, whereas XGBoost and $k$-NN deteriorate more markedly. The
neural network without time remains competitive with Lasso over land, but is
less competitive over ocean and shows larger run-to-run variability, especially
in later years.

Adding time does not materially improve XCO$_2$ prediction. Over ocean, the original and time-augmented Lasso curves in Figure~\ref{fig:OOT_xco2_land_ocean} are almost indistinguishable. Over land, the time-augmented Lasso is slightly worse in later years.
For the neural network, adding time worsens performance, especially over land. These results suggest that, for XCO$_2$, the spectra and
a priori inputs already capture much of the relevant long-term variation. An
explicit linear time feature therefore adds little and may over-correct in some
settings.

\subsubsection{Interpretation of the XCH$_4$ Results}

For XCH$_4$, all methods have comparable NRMSEs on the 2020 in-period test set,
although these errors are larger than for XCO$_2$. This suggests that XCH$_4$
is more difficult to emulate. On the 2021--2023 out-of-time test sets, however,
the errors of all models without the time feature increase rapidly. This includes
Lasso. Thus, for XCH$_4$, the relation learned from the 2020 data without an
explicit time feature does not remain stable in later years. In particular, the
approximately constant difference between retrieved and a priori XCH$_4$ during
the training period does not persist over time, as shown in
Figure~\ref{fig:Lasso_timeline_xch4_land_ocean}.

Adding time substantially improves XCH$_4$ prediction for Lasso and the neural
network. Their errors are much lower than those of the corresponding models
without time, and they increase much more mildly from 2021 to 2023. By contrast,
the effect of adding time is mixed for $k$-NN and XGBoost. This supports the
interpretation that the time feature compensates for the missing temporal trend
in the a priori XCH$_4$, as illustrated by the Lasso fits in
Figure~\ref{fig:Lasso_timeline_xch4_land_ocean}. Appendix~\ref{sec:learning-theory}
gives a formal interpretation in terms of a time-dependent residual component.

Among the time-augmented models, Lasso is the most stable. Over land, its errors
change very little from 2021 to 2023. The time-augmented neural network also
improves substantially, but its standard deviations are larger, especially in
later years. Thus, for XCH$_4$, the time-augmented Lasso combines low error with
greater stability across runs.

\subsubsection{Summary of the Results}

The main benefit of adding time is for XCH$_4$, not for XCO$_2$. For XCH$_4$,
both Lasso and the neural network improve when time is included, but Lasso gives
lower errors and smaller run-to-run variability overall. Since Lasso is also
simpler and more interpretable, these results motivate the more detailed analysis
of the Lasso emulator in the following section.

\section{Detailed Analysis of the Lasso Emulator}
\label{sec:detailed-analysis-Lasso-emulator}

We now focus on the Lasso emulator, which gave the most stable out-of-time
prediction performance in Section~\ref{sec:OOT-pred-performance}. Where relevant,
we compare the versions with and without the time feature; for the coefficient
analysis, we use the time-augmented Lasso.

Section~\ref{sec:Lasso-temporal-behaviour} examines the temporal behaviour of
the Lasso predictions and the corresponding global residual maps.
Section~\ref{sec:feature-importance} studies the fitted Lasso coefficients and
identifies the spectral and non-spectral inputs that drive the predictions.
Section~\ref{sec:TCCON-validation} validates the Lasso emulator against TCCON
and compares its errors with the GOSAT--TCCON discrepancy.

\subsection{Temporal and Spatial Behaviour of the Lasso}

\label{sec:Lasso-temporal-behaviour}

\begin{figure}[htp]
    \centering
    \includegraphics[width=0.9\linewidth]{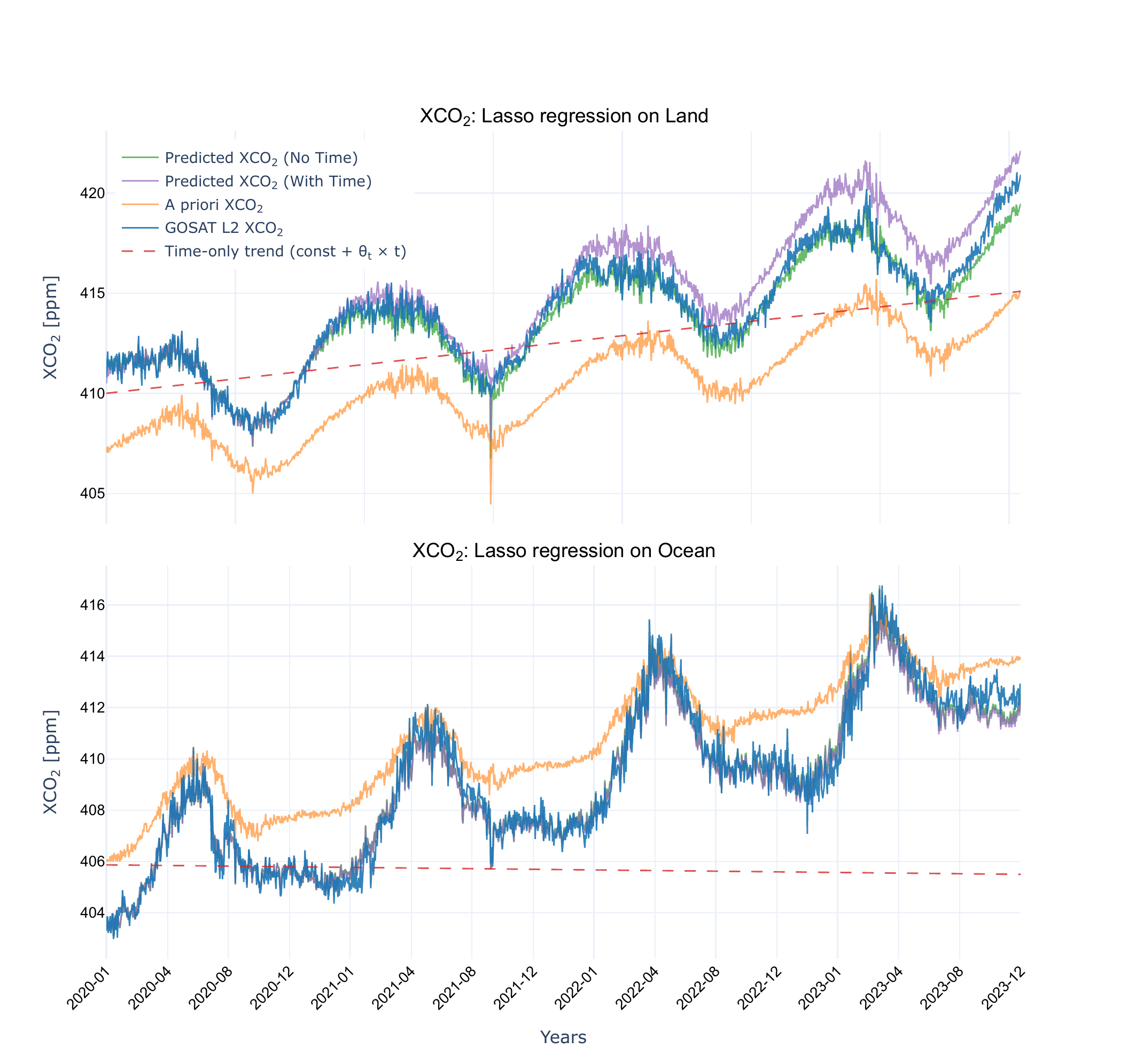}
\caption{Lasso predictions for XCO$_2$.
Predictions over land and ocean from 2020 to 2023 are shown. Predictions
without time (green) and with time (purple) are compared with the GOSAT
Level~2 retrievals (blue) and the a priori values (orange). The dashed red
line shows the fitted linear time component, $\mathrm{const}+\theta_t t$, from
the time-augmented Lasso. The models were trained on 2020 data using a random
80/20 train--test split and then evaluated on 2021--2023 data.}
    \label{fig:Lasso_timeline_xco2_land_ocean}
\end{figure}

\begin{figure}[htp]
    \centering
    \includegraphics[width=0.95\linewidth]{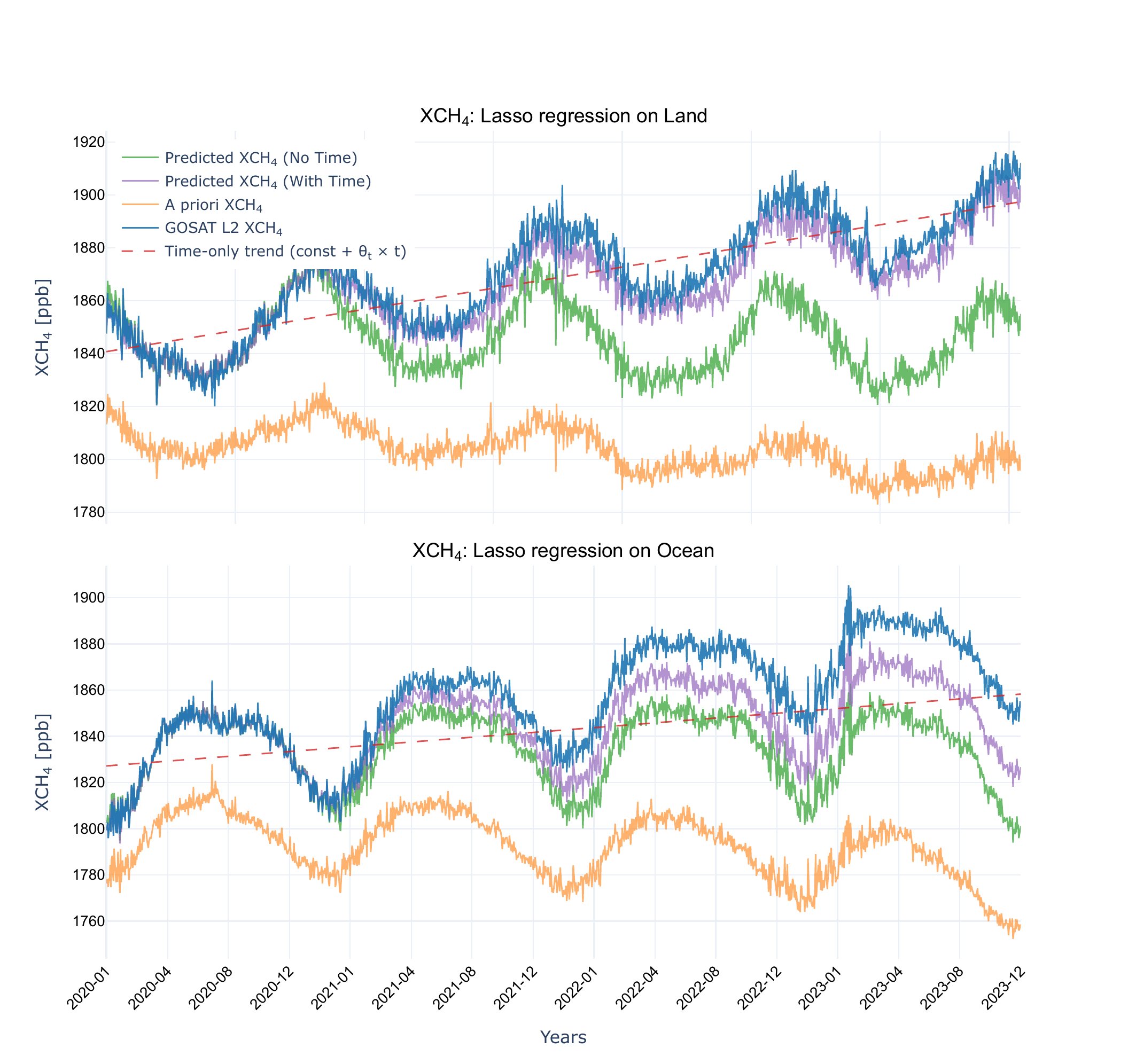}
\caption{Lasso predictions for XCH$_4$.
Predictions over land and ocean from 2020 to 2023 are shown. Predictions
without time (green) and with time (purple) are compared with the GOSAT
Level~2 retrievals (blue) and the a priori values (orange). The dashed red
line shows the fitted linear time component, $\mathrm{const}+\theta_t t$, from
the time-augmented Lasso. The models were trained on 2020 data using a random
80/20 train--test split and then evaluated on 2021--2023 data.}
    \label{fig:Lasso_timeline_xch4_land_ocean}
\end{figure}

\begin{figure}[t]
\centering
\includegraphics[width=0.95\textwidth]{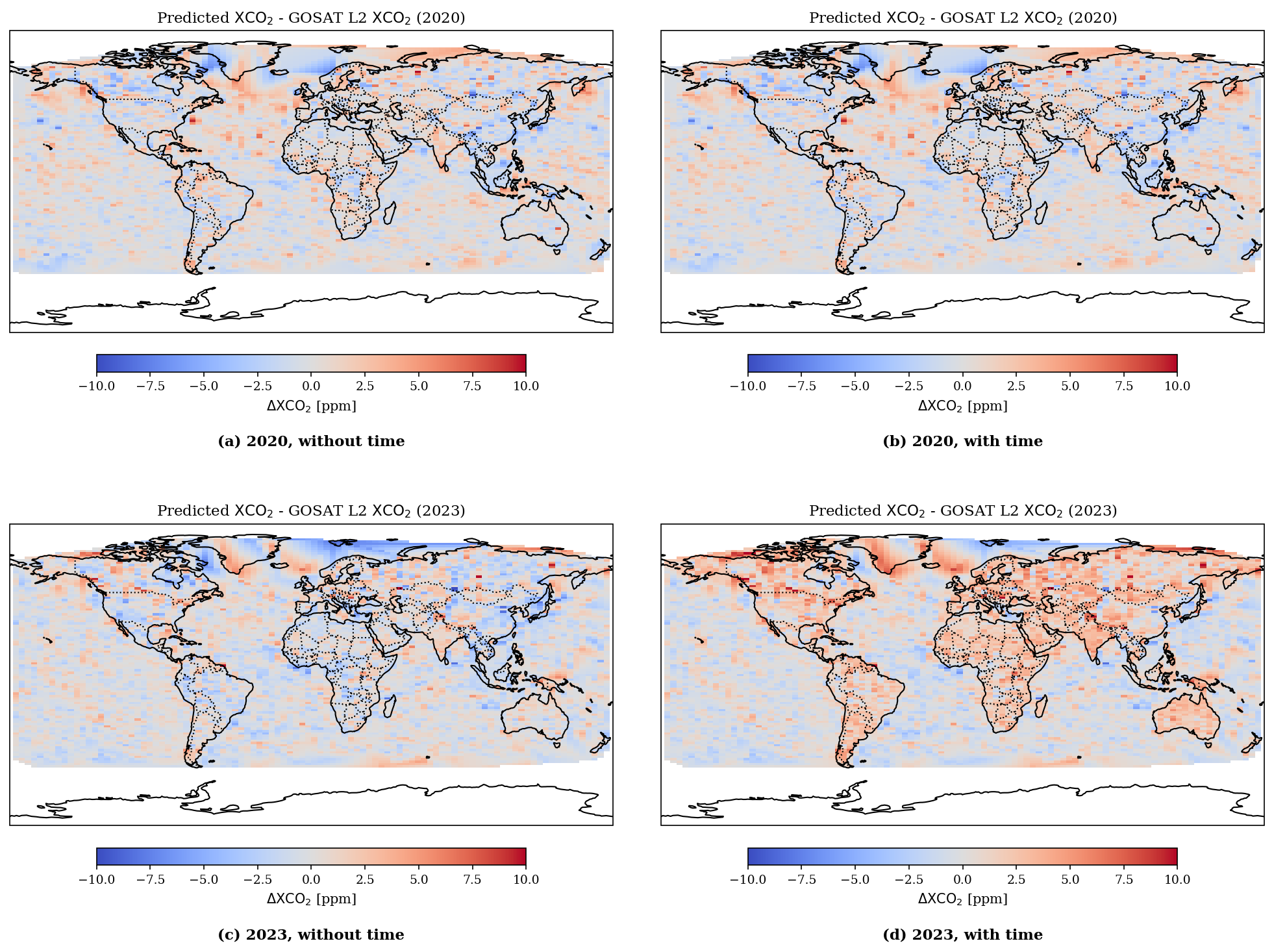}
\caption{Global residual maps for XCO$_2$.
Residual maps for 2020 and 2023 are shown for Lasso predictions without and
with the time feature. Residuals are defined as prediction minus GOSAT and are
aggregated on 2.5$^\circ$ grid boxes. The left column shows the model without
time, and the right column shows the time-augmented model. The top row shows
2020, and the bottom row shows 2023. Values are expressed in ppm.}
\label{fig:F10_GlobalResidual_co2_time_comparison}
\end{figure}

\begin{figure}[t]
\centering
\includegraphics[width=0.95\textwidth]{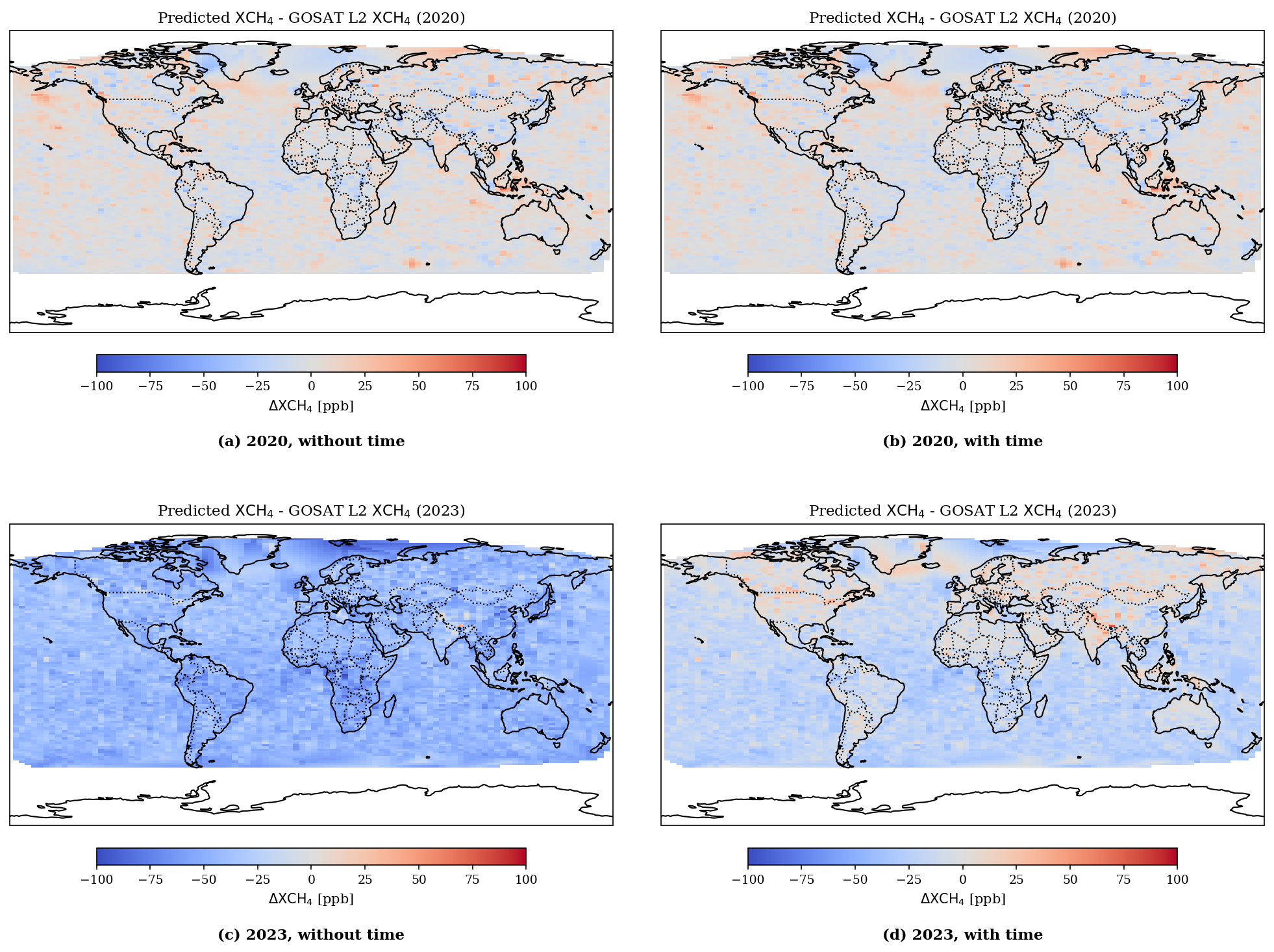}
\caption{Global residual maps for XCH$_4$.
Residual maps for 2020 and 2023 are shown for Lasso predictions without and
with the time feature. Residuals are defined as prediction minus GOSAT and are
aggregated on 2.5$^\circ$ grid boxes. The left column shows the model without
time, and the right column shows the time-augmented model. The top row shows
2020, and the bottom row shows 2023. Values are expressed in ppb.}
\label{fig:F11_GlobalResidual_ch4_time_comparison}
\end{figure}

We now examine the temporal and spatial behaviour of the Lasso in more detail.
Figures~\ref{fig:Lasso_timeline_xco2_land_ocean}
and~\ref{fig:Lasso_timeline_xch4_land_ocean}
show representative Lasso fits for XCO$_2$ and XCH$_4$, respectively, over land
and ocean, comparing the versions with and without the time feature. In each
case, the model was trained on a random 80\% subset of the 2020 data and then
evaluated on the remaining 2020 data and on the out-of-time data from 2021 to
2023. The figures also show the corresponding GOSAT Level~2 retrievals, the
a priori values, and, for the time-augmented model, the fitted linear time
component. These plots help explain why adding time has little effect for
XCO$_2$ but a substantial effect for XCH$_4$.

For XCO$_2$, the Lasso already tracks the retrievals well without the time
feature, both over land and over ocean. Adding time changes the predictions only
slightly. Over ocean, the two versions are almost indistinguishable throughout
the evaluation period. Over land, the time-augmented model shows a mild upward
drift in later years and becomes slightly less accurate by 2023. This is
consistent with the interpretation that, for XCO$_2$, the spectra and a priori
inputs already capture much of the relevant long-term variation. An additional
linear time term therefore adds little and may slightly over-correct in later
years.

For XCH$_4$, the behaviour is different. Without the time feature, the Lasso
fits the 2020 relation between the retrieval and the a priori XCH$_4$ reasonably
well, but increasingly underestimates the retrievals from 2021 onward, both over
land and over ocean. Adding time corrects much of this negative drift, but its
effect differs by surface type. Over land, the time-augmented Lasso remains
close to the GOSAT retrievals throughout 2021--2023. Over ocean, it reduces the
large underestimation of the model without time, but a negative residual remains,
especially in later years. Thus, for XCH$_4$, the time feature captures an
important temporal component missing from the other inputs, but a single linear
time term does not remove all surface-dependent discrepancies.

Figures~\ref{fig:F10_GlobalResidual_co2_time_comparison}
and~\ref{fig:F11_GlobalResidual_ch4_time_comparison}
provide a spatial counterpart to the time-series diagnostics. They show global
residual maps for 2020 and 2023, comparing the Lasso models without and with
the time feature. Residuals are defined as prediction minus GOSAT and are
aggregated on $2.5^\circ$ grid boxes.

For XCO$_2$, the residuals are small in 2020 for both models. In 2023, the
time-augmented model shows a broader positive residual pattern than the model
without time over land regions. This is consistent with the
time-series plots, where the time term leads to mild upward over-correction in
later years. Some regional residual structure remains, including high northern
latitudes and Greenland.

For XCH$_4$, the 2023 residual maps show a different pattern. Without time, the
model leaves a broad negative residual, consistent with the underestimation seen
in the time-series plots. Adding time reduces much of this negative residual,
especially over land. Over ocean, however, negative residuals remain in later
years. Thus, the time feature corrects an important part of the temporal drift
in XCH$_4$, but it does not remove all surface-dependent residual structure.

Taken together, the temporal and spatial diagnostics explain the main results in
Section~\ref{sec:prediction-results}. For XCO$_2$, explicit time information is
of limited value and can lead to mild over-correction in later years. For
XCH$_4$, the time feature is important because it removes much of the negative
drift seen without time, although the improvement is not uniform across surface
types. These results motivate the more detailed coefficient and TCCON analyses
below.

\subsection{Feature Importance}
\label{sec:feature-importance}

\begin{figure}[p]
    \centering
    \includegraphics[width=0.95\linewidth]{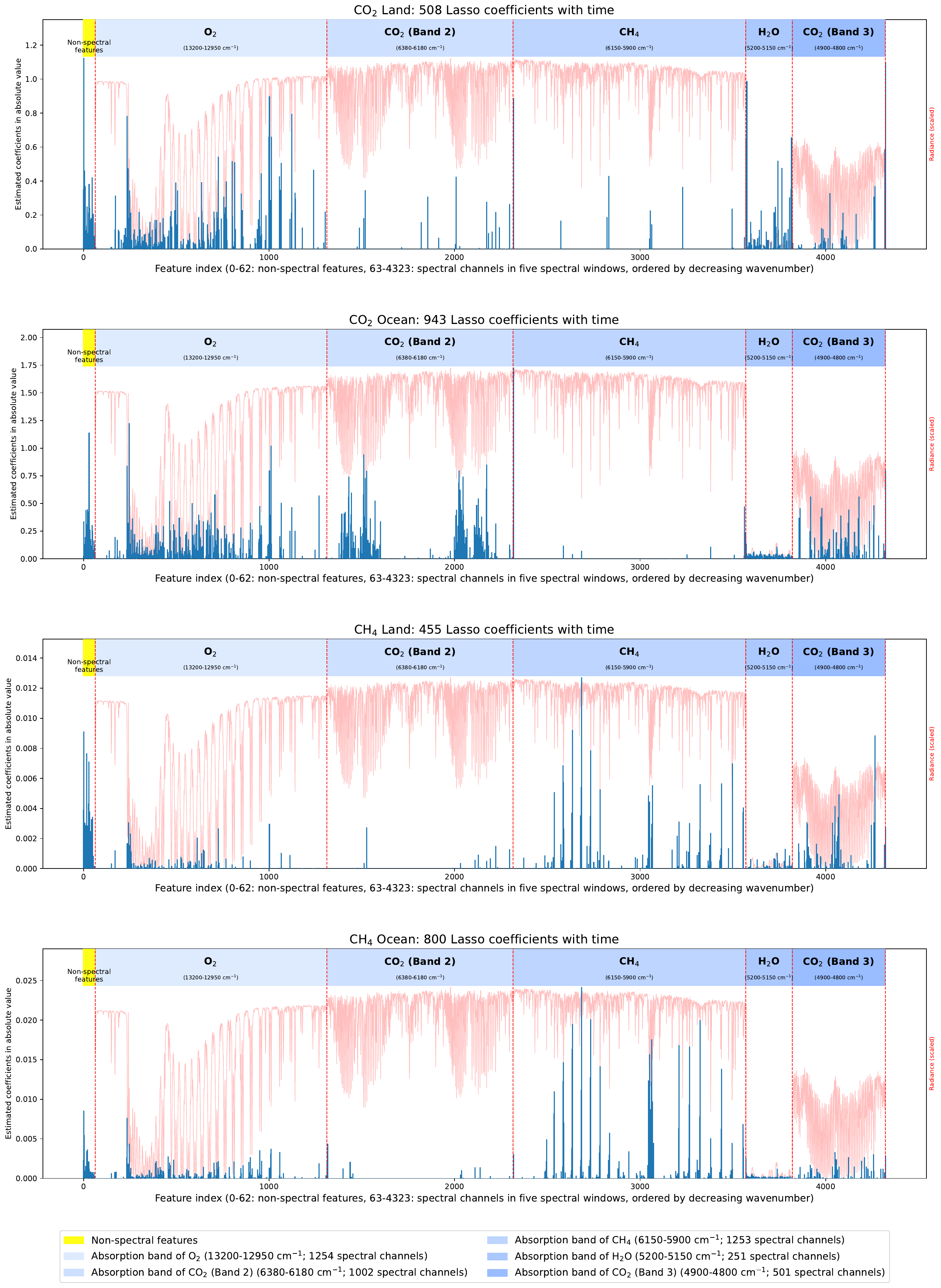}
\caption{Fitted Lasso coefficients and radiance spectra.
Absolute values of the fitted Lasso coefficients are shown in blue, together
with the corresponding radiance spectra in red. Here, ``non-spectral features''
denotes the input variables other than radiance channels, including the a priori
variables, geometry-related variables, Doppler variables, location and height,
surface pressure, temperature-profile variables, and time. See
Section~\ref{sec:feature-importance} for details.}
    \label{fig:Lasso-features-spectra}
\end{figure}

We analyze the coefficients from a single fit of the time-augmented Lasso. The aim is to identify which inputs are selected or weighted by the Lasso
emulator, not to explain how the retrieval algorithm itself uses the input features.
Figure~\ref{fig:Lasso-features-spectra} shows the nonzero coefficients from a model trained on a random 80\% subset of the 2020 data, separately for land and ocean, together with the corresponding radiance spectra.
Tables~\ref{tab:co2_lasso_time_top40_land_ocean} and~\ref{tab:ch4_lasso_time_top40_land_ocean} in Appendix~\ref{app:additional-tables} report the top 40 coefficients from the same fit.

In each panel of Figure~\ref{fig:Lasso-features-spectra}, the blue bars show the absolute values of the fitted Lasso coefficients, and the red curve shows the mean radiance spectrum, rescaled for visualization. From top to bottom, the panels correspond to XCO$_2$ over land, XCO$_2$ over ocean, XCH$_4$ over land, and XCH$_4$ over ocean. Coefficient magnitudes should be interpreted within each fit, not as quantities directly comparable across target gases or surface types.

The horizontal axis is the feature index and spans the 4324 input features of the time-augmented model.
Indices 0--62 correspond to non-spectral variables, including time, a priori profiles, geometry, and related quantities. Indices 63--4323 correspond to spectral radiance channels. These channels are sorted in decreasing wavenumber order and partitioned into five windows: O$_2$
(13200--12950~cm$^{-1}$), CO$_2$ (Band~2; 6380--6180~cm$^{-1}$), CH$_4$
(6150--5900~cm$^{-1}$), H$_2$O (5200--5150~cm$^{-1}$), and CO$_2$ (Band~3;
4900--4800~cm$^{-1}$). The dashed vertical lines mark the corresponding window limits in feature-index coordinates. Only nonzero coefficients are shown.

We now interpret Figure~\ref{fig:Lasso-features-spectra} case by case. The interpretation is about the fitted emulator: it identifies inputs useful for emulation, not how the retrieval algorithm itself uses the input features.

\noindent\textbf{XCO$_2$ over land.}
The largest contributions come from non-spectral input variables, the O$_2$
window, and selected channels in the H$_2$O and CO$_2$ (Band~3) windows.
The CO$_2$ window in Band~2 is less prominent. This suggests that, over land,
the fitted Lasso uses a combination of a priori variables and spectral channels
related not only to CO$_2$ absorption, but also to light-path and
dry-air-normalization effects. This is plausible because XCO$_2$ is a
column-averaged dry-air mole fraction, so accurate accounting for total air
mass, light-path modification by atmospheric particles, and water vapour is
important in addition to sensitivity to CO$_2$ absorption itself. This
interpretation is consistent with Table~\ref{tab:co2_lasso_time_top40_land_ocean},
where the top-ranked coefficients for XCO$_2$ over land include {\tt xco2\_ap}
and channels in the O$_2$, H$_2$O, and CO$_2$ (Band~3) windows.

\medskip

\noindent\textbf{XCO$_2$ over ocean.}
Over ocean, the O$_2$ window and the CO$_2$ windows in Bands~2 and~3 receive
nonzero coefficients over a broader range of channels, while the CH$_4$ window
is much less prominent. Compared with land, the prediction therefore appears to
rely more broadly on spectral channels, with the CO$_2$ window in Band~2
contributing in addition to the O$_2$ window and the CO$_2$ window in Band~3.

\medskip

\noindent\textbf{XCH$_4$ over land.}
The largest spectral contribution comes from the CH$_4$ window. Non-spectral
variables also receive nonzero coefficients, and the O$_2$ window and the CO$_2$ window in Band~3 make additional contributions. Thus, although the CH$_4$ window is dominant, the fitted model does not rely on it alone. The other selected windows may provide complementary information about light-path effects or retrieval conditions, rather than direct CH$_4$ absorption.

\medskip

\noindent\textbf{XCH$_4$ over ocean.}
Over ocean, the largest spectral contribution again comes from the CH$_4$
window. Non-spectral variables also receive nonzero coefficients, while the
O$_2$ window and the CO$_2$ window in Band~3 make additional contributions.
Thus, the fitted model relies mainly on the CH$_4$ absorption window, with
secondary contributions from auxiliary variables and other spectral windows.

Overall, Figure~\ref{fig:Lasso-features-spectra} shows that the Lasso does not
use the radiance spectrum uniformly. Instead, the nonzero coefficients are
concentrated on a subset of input features, and the relevant spectral windows
differ by target gas and surface type. For XCO$_2$, the selected input features
include non-spectral variables and channels in the O$_2$ and CO$_2$ windows. For
XCH$_4$, the largest spectral coefficients are concentrated in the CH$_4$
window, with additional contributions from non-spectral variables and other
spectral windows.
This suggests that the fitted Lasso selects spectral channels and auxiliary
input variables that are broadly consistent with the absorption bands and
auxiliary information relevant to greenhouse-gas retrieval, while remaining
sparse and interpretable.

Tables~\ref{tab:co2_lasso_time_top40_land_ocean} and~\ref{tab:ch4_lasso_time_top40_land_ocean} in Appendix~\ref{app:additional-tables} complement Figure~\ref{fig:Lasso-features-spectra} by listing the top-ranked individual features, including their coefficients and, for spectral channels, their
wavenumbers.

\noindent\textbf{Non-spectral features for XCO$_2$.}
For XCO$_2$ over land, the largest non-spectral coefficient is assigned to the column-mean a priori XCO$_2$, {\tt xco2\_ap}. Other non-spectral features, including level-wise CO$_2$ a priori variables, temperature variables, geometry variables, and time, also contribute, but with smaller coefficients. Over ocean, by contrast, the top-ranked coefficients are dominated by spectral channels, with only limited contributions from non-spectral variables such as solar-zenith geometry and latitude. This suggests that the fitted emulator uses a different balance of spectral and non-spectral information over land and ocean.

\noindent\textbf{Non-spectral features for XCH$_4$.}
For XCH$_4$ over land, non-spectral features contribute substantially. The fitted Lasso assigns relatively large coefficients to the column-mean a priori XCH$_4$, {\tt xch4\_ap}, a level-wise CH$_4$ a priori variable, Doppler velocity, geometry-related variables, and time. Thus, the emulator is not driven by the CH$_4$ window alone, even though that window provides the largest spectral contribution. Some channels in the CO$_2$ window in Band~3 also appear among the top-ranked coefficients, suggesting that they provide complementary information about observation or retrieval conditions rather than direct CH$_4$ absorption. Over ocean, by contrast, the largest coefficients are more concentrated in the CH$_4$ window, while non-spectral variables play secondary roles. This again suggests that the fitted emulator uses a different balance of spectral and non-spectral information over land and ocean.

\noindent\textbf{Time feature.}
The coefficient on time is better interpreted as part of the fitted temporal correction than as an ordinary variable-importance score. A modest coefficient on time therefore does not imply a minor role in out-of-time prediction. As shown above, the time term captures an important part of the temporal drift for XCH$_4$, although its effect differs by surface type. The main predictive information remains spectral and, in some cases, a priori.

\subsection{TCCON Validation} \label{sec:TCCON-validation}

We use TCCON validation to assess the Lasso emulator against an external
ground-based reference. This complements the comparison against the GOSAT
retrievals, which are the training targets. It also allows us to compare the
emulator--TCCON error with the GOSAT--TCCON discrepancy, and to assess whether
including time improves external validation performance.

For this validation, we refitted the Lasso models on the full 2020 data and
applied them to the TCCON-matched dataset. The regularization parameter
\(\lambda\) was not reselected; it was fixed at the value selected in the
corresponding 80\% training experiment (see Appendix~\ref{app:Lasso}).
Because TCCON is a ground-based network, the matched dataset is dominated by
land-based collocations. We therefore omit the ocean subset from the
surface-type summary, since too few ocean collocations were available for
reliable interpretation. Here and below, ``Other'' denotes collocated satellite
footprints that are not classified as pure land or pure ocean. The TCCON sites
used in this study are listed in Table~\ref{tab:tccon_sites}.

We report NRMSE, bias, and residual standard deviation. Bias and residual
standard deviation are reported in physical units: ppm for XCO$_2$ and ppb for
XCH$_4$. For each comparison, let
\[
e_i = y_i^{(1)} - y_i^{(2)}, \qquad i=1,\ldots,N,
\]
where \(N\) is the number of collocated samples, \(y_i^{(1)}\) is the quantity
being evaluated, and \(y_i^{(2)}\) is the reference quantity. Thus,
\(y_i^{(1)}\) is either the Lasso prediction or the GOSAT value, and
\(y_i^{(2)}\) is either the GOSAT value or the TCCON value, depending on the
comparison.

The NRMSE is computed with respect to the reference quantity:
\[
\mathrm{NRMSE}
=
\left\{
\frac{1}{N}\sum_{i=1}^{N}
\left(
\frac{y_i^{(1)}-y_i^{(2)}}{y_i^{(2)}}
\right)^2
\right\}^{1/2}.
\]
The bias and residual standard deviation are
\[
\mathrm{Bias}=\frac{1}{N}\sum_{i=1}^{N} e_i,
\qquad
\mathrm{STD}
=
\left\{
\frac{1}{N}\sum_{i=1}^{N}
\left(e_i-\mathrm{Bias}\right)^2
\right\}^{1/2}.
\]
The NRMSE results are summarized by year in
\autoref{tab:nrmse_yearly_pooled} and by surface type in
\autoref{tab:nrmse_surface_pooled}.
The corresponding bias and residual standard deviation values are reported by
year in Table~\ref{tab:bias_std_yearly}, and by surface type in
Table~\ref{tab:bias_std_surface}.

\begin{table}[h]
\centering
\caption{TCCON-matched NRMSE by year.
NRMSE is reported for the pooled 2021--2023 sample and separately by year.
The reference column indicates whether the error is computed against GOSAT or
TCCON. Here \(N\) denotes the number of matched GOSAT--TCCON samples.}
\label{tab:nrmse_yearly_pooled}

\normalsize
\renewcommand{\arraystretch}{1.15}
\setlength{\tabcolsep}{3.0pt}

\begin{tabular}{@{}c c l cc cc cc@{}}
\toprule
\textbf{Year} & \textbf{\(N\)} & \textbf{Reference}
& \multicolumn{2}{c}{\textbf{Lasso no time}}
& \multicolumn{2}{c}{\textbf{Lasso with time}}
& \multicolumn{2}{c}{\textbf{GOSAT}} \\
& & & \textbf{XCO$_2$} & \textbf{XCH$_4$}
& \textbf{XCO$_2$} & \textbf{XCH$_4$}
& \textbf{XCO$_2$} & \textbf{XCH$_4$} \\
\midrule
\multirow{2}{*}{2021--23} & \multirow{2}{*}{5225}
& GOSAT & 0.00614 & 0.01821 & 0.00582 & 0.00762 & -- & -- \\
& & TCCON & 0.00490 & 0.01796 & 0.00367 & 0.00646 & 0.00581 & 0.00712 \\
\midrule
\multirow{2}{*}{2021} & \multirow{2}{*}{2201}
& GOSAT & 0.00600 & 0.01201 & 0.00547 & 0.00741 & -- & -- \\
& & TCCON & 0.00475 & 0.01191 & 0.00381 & 0.00639 & 0.00595 & 0.00702 \\
\midrule
\multirow{2}{*}{2022} & \multirow{2}{*}{1671}
& GOSAT & 0.00618 & 0.01907 & 0.00564 & 0.00801 & -- & -- \\
& & TCCON & 0.00500 & 0.01897 & 0.00303 & 0.00709 & 0.00573 & 0.00731 \\
\midrule
\multirow{2}{*}{2023} & \multirow{2}{*}{1353}
& GOSAT & 0.00630 & 0.02444 & 0.00656 & 0.00747 & -- & -- \\
& & TCCON & 0.00503 & 0.02388 & 0.00413 & 0.00572 & 0.00568 & 0.00702 \\
\bottomrule
\end{tabular}
\end{table}

\begin{table}[h]
\centering
\caption{TCCON-matched NRMSE by surface type.
NRMSE is reported for the pooled 2021--2023 period. The ocean subset is omitted
because too few collocations were available for reliable interpretation. The
reference column specifies whether the error is computed relative to GOSAT or
TCCON. Here \(N\) denotes the number of matched GOSAT--TCCON samples in each
surface category.}
\label{tab:nrmse_surface_pooled}

\normalsize
\renewcommand{\arraystretch}{1.15}
\setlength{\tabcolsep}{3pt}

\begin{tabular}{c c l c c c c c c}
\toprule
\textbf{Surface} & \textbf{\(N\)} & \textbf{Reference}
& \multicolumn{2}{c}{\textbf{Lasso no time}}
& \multicolumn{2}{c}{\textbf{Lasso with time}}
& \multicolumn{2}{c}{\textbf{GOSAT}} \\
& &
& \textbf{XCO$_2$} & \textbf{XCH$_4$}
& \textbf{XCO$_2$} & \textbf{XCH$_4$}
& \textbf{XCO$_2$} & \textbf{XCH$_4$} \\
\midrule
\multirow{2}{*}{All} & \multirow{2}{*}{5225}
& GOSAT & 0.00614 & 0.01821 & 0.00582 & 0.00762 & -- & -- \\
& & TCCON & 0.00490 & 0.01796 & 0.00367 & 0.00646 & 0.00581 & 0.00712 \\
\midrule
\multirow{2}{*}{Land} & \multirow{2}{*}{4178}
& GOSAT & 0.00610 & 0.01836 & 0.00550 & 0.00762 & -- & -- \\
& & TCCON & 0.00511 & 0.01805 & 0.00358 & 0.00654 & 0.00566 & 0.00713 \\
\midrule
\multirow{2}{*}{Other} & \multirow{2}{*}{1047}
& GOSAT & 0.00626 & 0.01763 & 0.00696 & 0.00762 & -- & -- \\
& & TCCON & 0.00400 & 0.01759 & 0.00400 & 0.00614 & 0.00638 & 0.00708 \\
\bottomrule
\end{tabular}
\end{table}

\begin{table}[htbp]
\centering
\caption{TCCON-matched bias and residual standard deviation by year.
Results are reported for the pooled 2021--2023 sample and separately by year.
Each entry gives the bias, with the residual standard deviation in parentheses.
Units are ppm for XCO$_2$ and ppb for XCH$_4$. The reference column specifies
whether the error is computed relative to GOSAT or TCCON. Here \(N\) denotes
the number of matched GOSAT--TCCON samples.}
\label{tab:bias_std_yearly}

\small
\renewcommand{\arraystretch}{1.08}
\setlength{\tabcolsep}{4pt}

\begin{tabular}{c c l c c c}
\toprule
\multicolumn{6}{c}{\normalsize\textbf{Panel A: XCO$_2$}} \\
\hline
\textbf{Year} & \textbf{\(N\)} & \textbf{Reference}
& \textbf{Lasso no time}
& \textbf{Lasso with time}
& \textbf{GOSAT} \\
\hline
2021--23 & 5225 & GOSAT
& \begin{tabular}[c]{@{}c@{}}-1.0756\\[0.15ex](2.3295)\end{tabular}
& \begin{tabular}[c]{@{}c@{}}0.2746\\[0.15ex](2.4080)\end{tabular}
& -- \\[0.65ex]
2021--23 & 5225 & TCCON
& \begin{tabular}[c]{@{}c@{}}-1.4570\\[0.15ex](1.4413)\end{tabular}
& \begin{tabular}[c]{@{}c@{}}-0.1068\\[0.15ex](1.5286)\end{tabular}
& \begin{tabular}[c]{@{}c@{}}-0.3814\\[0.15ex](2.3918)\end{tabular} \\[0.65ex]
\hline
2021 & 2201 & GOSAT
& \begin{tabular}[c]{@{}c@{}}-1.1275\\[0.15ex](2.2356)\end{tabular}
& \begin{tabular}[c]{@{}c@{}}-0.3780\\[0.15ex](2.2455)\end{tabular}
& -- \\[0.65ex]
2021 & 2201 & TCCON
& \begin{tabular}[c]{@{}c@{}}-1.3231\\[0.15ex](1.4695)\end{tabular}
& \begin{tabular}[c]{@{}c@{}}-0.5736\\[0.15ex](1.4781)\end{tabular}
& \begin{tabular}[c]{@{}c@{}}-0.1956\\[0.15ex](2.4640)\end{tabular} \\[0.65ex]
\hline
2022 & 1671 & GOSAT
& \begin{tabular}[c]{@{}c@{}}-1.1207\\[0.15ex](2.3331)\end{tabular}
& \begin{tabular}[c]{@{}c@{}}0.3471\\[0.15ex](2.3247)\end{tabular}
& -- \\[0.65ex]
2022 & 1671 & TCCON
& \begin{tabular}[c]{@{}c@{}}-1.6710\\[0.15ex](1.2542)\end{tabular}
& \begin{tabular}[c]{@{}c@{}}-0.2032\\[0.15ex](1.2503)\end{tabular}
& \begin{tabular}[c]{@{}c@{}}-0.5503\\[0.15ex](2.3252)\end{tabular} \\[0.65ex]
\hline
2023 & 1353 & GOSAT
& \begin{tabular}[c]{@{}c@{}}-0.9353\\[0.15ex](2.4651)\end{tabular}
& \begin{tabular}[c]{@{}c@{}}1.2467\\[0.15ex](2.4232)\end{tabular}
& -- \\[0.65ex]
2023 & 1353 & TCCON
& \begin{tabular}[c]{@{}c@{}}-1.4105\\[0.15ex](1.5745)\end{tabular}
& \begin{tabular}[c]{@{}c@{}}0.7716\\[0.15ex](1.5495)\end{tabular}
& \begin{tabular}[c]{@{}c@{}}-0.4752\\[0.15ex](2.3321)\end{tabular} \\
\toprule
\multicolumn{6}{c}{\normalsize\textbf{Panel B: XCH$_4$}} \\
\hline
\textbf{Year} & \textbf{\(N\)} & \textbf{Reference}
& \textbf{Lasso no time}
& \textbf{Lasso with time}
& \textbf{GOSAT} \\
\hline
2021--23 & 5225 & GOSAT
& \begin{tabular}[c]{@{}c@{}}-30.1307\\[0.15ex](17.0636)\end{tabular}
& \begin{tabular}[c]{@{}c@{}}-4.7126\\[0.15ex](13.6949)\end{tabular}
& -- \\[0.65ex]
2021--23 & 5225 & TCCON
& \begin{tabular}[c]{@{}c@{}}-30.6987\\[0.15ex](14.6443)\end{tabular}
& \begin{tabular}[c]{@{}c@{}}-5.2806\\[0.15ex](11.0056)\end{tabular}
& \begin{tabular}[c]{@{}c@{}}-0.5680\\[0.15ex](13.4254)\end{tabular} \\[0.65ex]
\hline
2021 & 2201 & GOSAT
& \begin{tabular}[c]{@{}c@{}}-18.6757\\[0.15ex](12.9798)\end{tabular}
& \begin{tabular}[c]{@{}c@{}}-4.7214\\[0.15ex](13.2029)\end{tabular}
& -- \\[0.65ex]
2021 & 2201 & TCCON
& \begin{tabular}[c]{@{}c@{}}-19.8573\\[0.15ex](10.5254)\end{tabular}
& \begin{tabular}[c]{@{}c@{}}-5.9030\\[0.15ex](10.5217)\end{tabular}
& \begin{tabular}[c]{@{}c@{}}-1.1816\\[0.15ex](13.1606)\end{tabular} \\[0.65ex]
\hline
2022 & 1671 & GOSAT
& \begin{tabular}[c]{@{}c@{}}-33.5130\\[0.15ex](13.6683)\end{tabular}
& \begin{tabular}[c]{@{}c@{}}-6.1390\\[0.15ex](13.9758)\end{tabular}
& -- \\[0.65ex]
2022 & 1671 & TCCON
& \begin{tabular}[c]{@{}c@{}}-34.0497\\[0.15ex](11.2250)\end{tabular}
& \begin{tabular}[c]{@{}c@{}}-6.6758\\[0.15ex](11.5978)\end{tabular}
& \begin{tabular}[c]{@{}c@{}}-0.5367\\[0.15ex](13.7872)\end{tabular} \\[0.65ex]
\hline
2023 & 1353 & GOSAT
& \begin{tabular}[c]{@{}c@{}}-44.5881\\[0.15ex](13.5286)\end{tabular}
& \begin{tabular}[c]{@{}c@{}}-2.9367\\[0.15ex](13.9216)\end{tabular}
& -- \\[0.65ex]
2023 & 1353 & TCCON
& \begin{tabular}[c]{@{}c@{}}-44.1965\\[0.15ex](10.2514)\end{tabular}
& \begin{tabular}[c]{@{}c@{}}-2.5452\\[0.15ex](10.5366)\end{tabular}
& \begin{tabular}[c]{@{}c@{}}0.3915\\[0.15ex](13.3414)\end{tabular} \\
\hline
\end{tabular}
\end{table}

\begin{table}[htbp]
\centering
\caption{TCCON-matched bias and residual standard deviation by surface type.
Results are reported for the pooled 2021--2023 period. Each entry gives the
bias, with the residual standard deviation in parentheses. Units are ppm for
XCO$_2$ and ppb for XCH$_4$. The reference column specifies whether the error
is computed relative to GOSAT or TCCON. Here \(N\) denotes the number of
matched GOSAT--TCCON samples in each surface category.}
\label{tab:bias_std_surface}

\small
\renewcommand{\arraystretch}{1.08}
\setlength{\tabcolsep}{4pt}

\begin{tabular}{c c l c c c}
\toprule
\multicolumn{6}{c}{\normalsize\textbf{Panel A: XCO$_2$}} \\
\hline
\textbf{Surface} & \textbf{\(N\)} & \textbf{Reference}
& \textbf{Lasso no time}
& \textbf{Lasso with time}
& \textbf{GOSAT} \\
\hline
All & 5225 & GOSAT
& \begin{tabular}[c]{@{}c@{}}-1.0756\\[0.15ex](2.3295)\end{tabular}
& \begin{tabular}[c]{@{}c@{}}0.2746\\[0.15ex](2.4080)\end{tabular}
& -- \\[0.65ex]
All & 5225 & TCCON
& \begin{tabular}[c]{@{}c@{}}-1.4570\\[0.15ex](1.4413)\end{tabular}
& \begin{tabular}[c]{@{}c@{}}-0.1068\\[0.15ex](1.5286)\end{tabular}
& \begin{tabular}[c]{@{}c@{}}-0.3814\\[0.15ex](2.3918)\end{tabular} \\[0.65ex]
\hline
Land & 4178 & GOSAT
& \begin{tabular}[c]{@{}c@{}}-1.2703\\[0.15ex](2.2182)\end{tabular}
& \begin{tabular}[c]{@{}c@{}}0.1111\\[0.15ex](2.2900)\end{tabular}
& -- \\[0.65ex]
Land & 4178 & TCCON
& \begin{tabular}[c]{@{}c@{}}-1.6342\\[0.15ex](1.3730)\end{tabular}
& \begin{tabular}[c]{@{}c@{}}-0.2528\\[0.15ex](1.4751)\end{tabular}
& \begin{tabular}[c]{@{}c@{}}-0.3639\\[0.15ex](2.3318)\end{tabular} \\[0.65ex]
\hline
Other & 1047 & GOSAT
& \begin{tabular}[c]{@{}c@{}}-0.2985\\[0.15ex](2.5870)\end{tabular}
& \begin{tabular}[c]{@{}c@{}}0.9272\\[0.15ex](2.7349)\end{tabular}
& -- \\[0.65ex]
Other & 1047 & TCCON
& \begin{tabular}[c]{@{}c@{}}-0.7498\\[0.15ex](1.4898)\end{tabular}
& \begin{tabular}[c]{@{}c@{}}0.4759\\[0.15ex](1.5980)\end{tabular}
& \begin{tabular}[c]{@{}c@{}}-0.4513\\[0.15ex](2.6163)\end{tabular} \\
\toprule
\multicolumn{6}{c}{\normalsize\textbf{Panel B: XCH$_4$}} \\
\hline
\textbf{Surface} & \textbf{\(N\)} & \textbf{Reference}
& \textbf{Lasso no time}
& \textbf{Lasso with time}
& \textbf{GOSAT} \\
\hline
All & 5225 & GOSAT
& \begin{tabular}[c]{@{}c@{}}-30.1307\\[0.15ex](17.0636)\end{tabular}
& \begin{tabular}[c]{@{}c@{}}-4.7126\\[0.15ex](13.6949)\end{tabular}
& -- \\[0.65ex]
All & 5225 & TCCON
& \begin{tabular}[c]{@{}c@{}}-30.6987\\[0.15ex](14.6443)\end{tabular}
& \begin{tabular}[c]{@{}c@{}}-5.2806\\[0.15ex](11.0056)\end{tabular}
& \begin{tabular}[c]{@{}c@{}}-0.5680\\[0.15ex](13.4254)\end{tabular} \\[0.65ex]
\hline
Land & 4178 & GOSAT
& \begin{tabular}[c]{@{}c@{}}-30.4753\\[0.15ex](17.0063)\end{tabular}
& \begin{tabular}[c]{@{}c@{}}-5.1905\\[0.15ex](13.5325)\end{tabular}
& -- \\[0.65ex]
Land & 4178 & TCCON
& \begin{tabular}[c]{@{}c@{}}-30.9662\\[0.15ex](14.4398)\end{tabular}
& \begin{tabular}[c]{@{}c@{}}-5.6814\\[0.15ex](10.9684)\end{tabular}
& \begin{tabular}[c]{@{}c@{}}-0.4909\\[0.15ex](13.4450)\end{tabular} \\[0.65ex]
\hline
Other & 1047 & GOSAT
& \begin{tabular}[c]{@{}c@{}}-28.7559\\[0.15ex](17.2220)\end{tabular}
& \begin{tabular}[c]{@{}c@{}}-2.8059\\[0.15ex](14.1649)\end{tabular}
& -- \\[0.65ex]
Other & 1047 & TCCON
& \begin{tabular}[c]{@{}c@{}}-29.6315\\[0.15ex](15.3873)\end{tabular}
& \begin{tabular}[c]{@{}c@{}}-3.6816\\[0.15ex](11.0083)\end{tabular}
& \begin{tabular}[c]{@{}c@{}}-0.8757\\[0.15ex](13.3422)\end{tabular} \\
\hline
\end{tabular}
\end{table}

\subsubsection{Effect of the Time Feature}

We first compare the Lasso models with and without the time feature, using
TCCON as the reference. The main effect is for XCH$_4$. In the pooled
2021--2023 sample, the NRMSE against TCCON decreases from 0.01796 without time
to 0.00646 with time. The decrease is present in each year. Thus, for
XCH$_4$, the time feature substantially reduces the TCCON validation error
left by the model without time.

For XCO$_2$, the effect is smaller. In the pooled 2021--2023 sample, the
NRMSE against TCCON decreases from 0.00490 without time to 0.00367 with time.
The yearly NRMSE values also decrease, but the gain is much smaller than for
XCH$_4$. The surface-type breakdown gives the same qualitative result: the
main NRMSE improvement from including time is for XCH$_4$.

The bias results support this conclusion. In the pooled 2021--2023 sample, the
XCH$_4$ bias against TCCON changes from \(-30.70\) ppb without time to
\(-5.28\) ppb with time. For XCO$_2$, the corresponding change is from
\(-1.46\) ppm to \(-0.11\) ppm. The absolute XCH$_4$ bias is reduced in each
year. For XCO$_2$, the absolute bias is reduced in 2021 and 2022, but not in
2023, where the bias changes sign.

The residual standard deviations show a different pattern. For XCH$_4$, the
pooled value against TCCON decreases from 14.64 ppb without time to 11.01 ppb
with time, but the year-by-year values do not decrease uniformly. The pooled
decrease is therefore partly due to smaller year-specific biases, not to a
uniform reduction of within-year residual variability. For XCO$_2$, the pooled
value changes only slightly, from 1.44 ppm to 1.53 ppm. Thus, for XCO$_2$, the
time feature mainly reduces bias.

\subsubsection{Reference-dependent Errors of the Time-augmented Lasso}

We now fix the time-augmented Lasso and compare its errors with respect to two
reference quantities: GOSAT and TCCON. This separates the error relative to the
training target from the error relative to the external ground-based reference.

For both target variables, the NRMSE is smaller when TCCON is used as the
reference. In the pooled 2021--2023 sample, the NRMSE against TCCON and
against GOSAT is 0.00367 and 0.00582 for XCO$_2$, and 0.00646 and 0.00762
for XCH$_4$, respectively. The same ordering holds in each year and in both
reported surface categories.

The residual standard deviations show the same ordering. In the pooled
2021--2023 sample, the residual standard deviation against TCCON and against
GOSAT is 1.53 ppm and 2.41 ppm for XCO$_2$, and 11.01 ppb and 13.69 ppb for
XCH$_4$, respectively. The same ordering again holds year by year and by
surface category.

The bias results do not show the same ordering. For XCO$_2$, the pooled bias is
closer to zero against TCCON than against GOSAT: \(-0.11\) ppm versus
\(0.27\) ppm. For XCH$_4$, it is slightly farther from zero against TCCON:
\(-5.28\) ppb versus \(-4.71\) ppb. Thus, the smaller NRMSE against TCCON
cannot be attributed to smaller bias. It is more consistently associated with
smaller residual standard deviation.

The smaller NRMSE against TCCON is notable because the emulator is trained on
GOSAT retrievals, not on TCCON. The main numerical reason is the smaller
residual standard deviation: the Lasso--TCCON residuals vary less around their
mean than the Lasso--GOSAT residuals. A plausible explanation is the limited
flexibility of the Lasso emulator. As a restricted linear model, it need not
reproduce all observation-specific variation in the GOSAT retrievals.
If part of that variation is not shared by TCCON, the Lasso prediction can have
a smaller residual standard deviation against TCCON than against GOSAT on the
TCCON-matched dataset.
The next subsection compares the emulator--TCCON errors with the
GOSAT--TCCON discrepancy.

\subsubsection{Comparison with the GOSAT--TCCON Discrepancy}

We finally compare the time-augmented Lasso--TCCON error with the
GOSAT--TCCON discrepancy. This comparison asks whether the emulator error
against TCCON is of the same order as the GOSAT--TCCON difference on the same
TCCON-matched dataset.

In the pooled 2021--2023 TCCON-matched dataset, the time-augmented Lasso has
smaller NRMSE against TCCON than GOSAT for both targets. For XCO$_2$, the
NRMSE is 0.00367 for the time-augmented Lasso and 0.00581 for GOSAT. For
XCH$_4$, the corresponding values are 0.00646 and 0.00712. The same ordering
holds in each year. Thus, the time-augmented Lasso--TCCON NRMSE is of the same
order as, and in these comparisons smaller than, the GOSAT--TCCON NRMSE.

The bias and residual standard deviation clarify this comparison. For XCO$_2$,
the pooled absolute bias and residual standard deviation are both smaller for
the time-augmented Lasso than for GOSAT: the bias is \(-0.11\) ppm versus
\(-0.38\) ppm, and the residual standard deviation is 1.53 ppm versus
2.39 ppm. The residual standard deviation is also smaller in each year. The
bias comparison, however, is not uniform year by year.

For XCH$_4$, the comparison is different. The time-augmented Lasso has a larger
pooled absolute bias than GOSAT, 5.28 ppb versus 0.57 ppb, but a smaller pooled
residual standard deviation, 11.01 ppb versus 13.43 ppb. The same pattern holds
in each year: GOSAT has the smaller absolute bias, while the time-augmented
Lasso has the smaller residual standard deviation. Thus, for XCH$_4$, the
smaller NRMSE of the emulator against TCCON is mainly due to its smaller
residual standard deviation, not to smaller bias.

The two-dimensional histograms in \autoref{fig:TCCON_comparison} provide a visual check of this comparison. For
XCO$_2$, the time-augmented Lasso has both smaller absolute bias and smaller
residual standard deviation against TCCON than GOSAT. For XCH$_4$, GOSAT has
smaller absolute bias, whereas the time-augmented Lasso has smaller residual
standard deviation. Thus, the histograms are consistent with the numerical
comparison above.

This comparison should be interpreted with care. It does not imply that the
emulator is a more accurate estimate of the true column abundance than the
retrieval algorithm. The same caution applies to the reference-dependent
comparison above, where the time-augmented Lasso has smaller errors against
TCCON than against GOSAT. The Lasso is trained on GOSAT retrievals and, as a
restricted linear model, need not reproduce all observation-specific variation
in those retrievals. If part of that variation is not shared by TCCON, the
emulator can have a smaller residual standard deviation against TCCON than
against GOSAT on the TCCON-matched dataset. The conclusion is therefore
limited: the time-augmented Lasso gives TCCON validation errors of the same
order as the GOSAT--TCCON discrepancy, with slightly smaller NRMSE in these
comparisons, subject to the usual limitations of collocation-based validation.

\begin{figure}[!htb]
    \centering
    \includegraphics[width=0.95\linewidth]{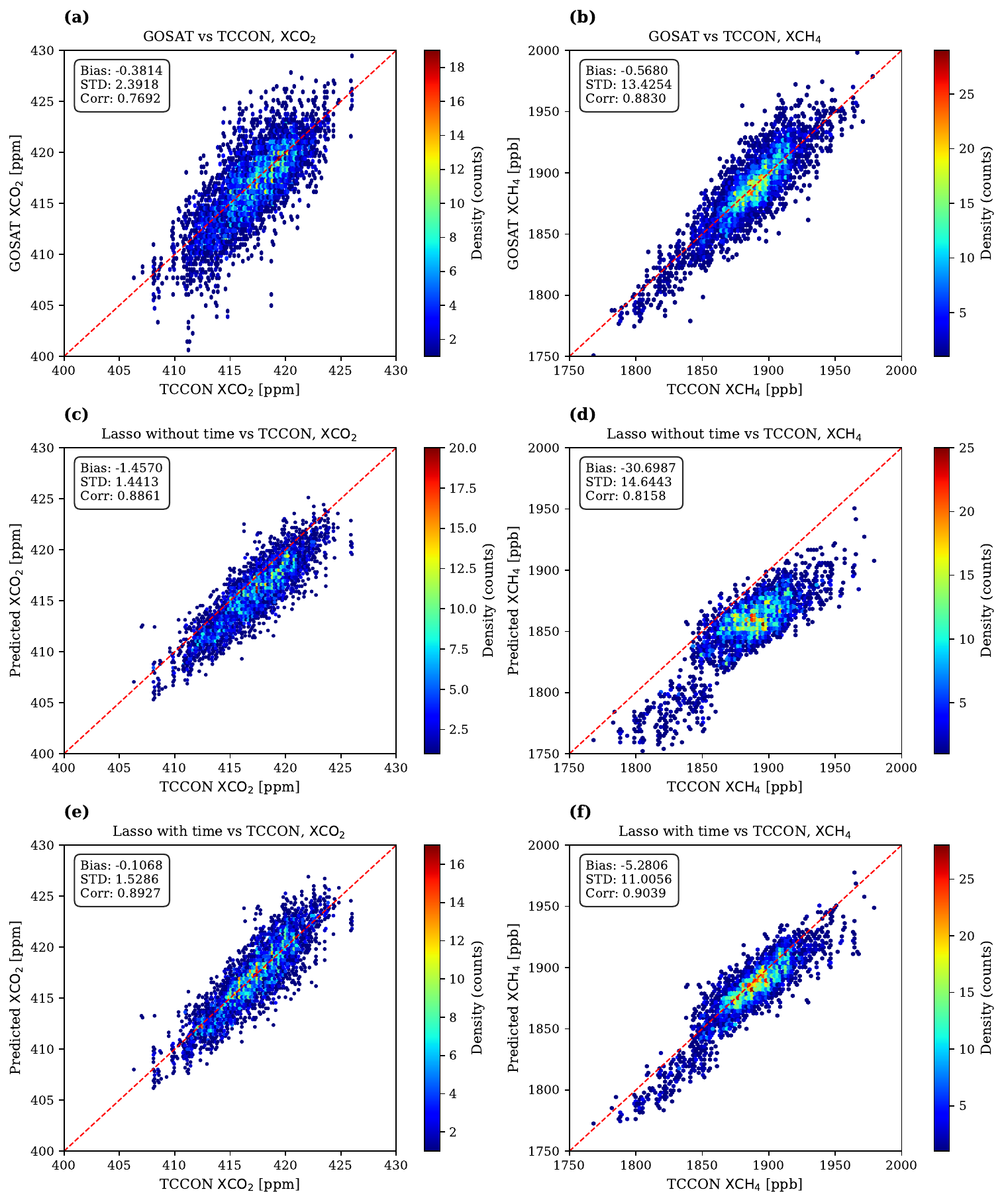}
    \caption{TCCON comparisons for GOSAT retrievals and Lasso predictions.
    Two-dimensional histograms compare GOSAT retrievals and Lasso predictions
    with TCCON values on the 2021--2023 TCCON-matched dataset (\(n=5225\)).
    The left column shows XCO$_2$, and the right column shows XCH$_4$.
    The rows show GOSAT retrievals, Lasso predictions without the time feature,
    and Lasso predictions with the time feature, respectively. The dashed red
    line indicates the one-to-one relationship.}
    \label{fig:TCCON_comparison}
\end{figure}

\clearpage
\section{Conclusion and Future Directions}

This paper studied whether machine-learning emulators of satellite greenhouse-gas retrievals remain accurate when they are trained on past observations and applied to later observations.
This is the operationally relevant setting, but it is not well assessed by random train--test splits drawn from the same period, as used in much of the previous related work. Using GOSAT retrievals for XCO$_2$ and XCH$_4$, we compared several emulators under this later-period prediction setting, examined the effect of adding observation time as an input feature, identified the time-augmented Lasso as the most stable method over time, and validated it against TCCON.

The results show that good in-period accuracy does not guarantee accurate
later-period prediction. Prediction accuracy generally deteriorates as the test
period moves away from the training period, especially for XCH$_4$. Adding a
time feature strongly improves XCH$_4$ prediction for the Lasso and
neural-network models, while its effect on XCO$_2$ is smaller and less uniform.
Among the methods considered, the time-augmented Lasso gives the most stable
prediction performance overall. On the TCCON-matched dataset, its validation
errors are of the same order as the GOSAT--TCCON discrepancy, suggesting that
it can serve as a practical surrogate for the retrieval algorithm, subject to
the usual limitations of collocation-based validation.

A plausible explanation for the stable performance of the Lasso is that the
part of the retrieval mapping relevant to this emulation task is well
approximated by a sparse linear function of the available input features. The
spectral information used by the fitted model is not spread uniformly over all
radiance channels. Instead, nonzero coefficients are concentrated at selected
wavenumbers within the absorption bands, together with selected a priori and
geometry-related variables. This provides one possible explanation for the
observed stability of the Lasso under later-period testing. Thus, the success
of the Lasso does not show that the full retrieval problem is linear. It
suggests only that, for the present emulation task, a sparse linear
approximation is accurate enough and more stable than the more flexible
alternatives considered here.

These findings suggest a possible route toward operational use. An emulator
trained on historical full-physics retrievals could be applied to newly acquired
satellite observations to provide fast preliminary estimates of XCO$_2$ and
XCH$_4$. The Lasso is attractive for this purpose because it is computationally
cheap, easy to retrain as new retrievals become available, and interpretable
through its fitted coefficients. Such an emulator would not replace the full
retrieval algorithm or ground-based validation. Rather, it could provide a
first-pass estimate when rapid concentration estimates are needed. For
operational use, further work is needed on uncertainty quantification,
monitoring of temporal drift, regular retraining, and continued validation as
new data become available.

\newpage

\appendix

\section{A Population-level View of Out-of-time Deterioration and Time Augmentation} \label{sec:learning-theory}
This appendix gives a population-level explanation of why out-of-time prediction may deteriorate and why adding time can improve it.

For each period \(k\in\{0,1,2,3\}\), let \(P_k\) denote the joint distribution of \((X,T)\), where \(X\) is the original emulator input vector excluding time, and \(T\) is the normalized observation time. Here \(k=0,1,2,3\) correspond to 2020, 2021, 2022, and 2023, respectively. We write \(E_k[\cdot]\) for expectation under \(P_k\).

The retrieval output is
\[
Y=f^*(X),
\]
where \(f^*\) is the fixed retrieval map. Thus, the retrieval map itself does not change across periods; what changes is the input distribution \(P_k\).

Time \(T\) is not an input to the retrieval algorithm itself. It is included only as an additional feature for the emulator. Therefore, adding time does not change the target \(Y\). It only enlarges the class of predictors used to approximate the same fixed map \(f^*\).

\subsection{Population risk and out-of-time deterioration}

Let
\[
\mathcal F=\{f_\theta:\theta\in\Theta\}
\]
be a class of predictors, where \(f_\theta\) may depend on both \(X\) and \(T\).
For period \(k\), define the population risk
\[
R_k(\theta):=E_k\!\left[\bigl(Y-f_\theta(X,T)\bigr)^2\right].
\]
This is the mean squared prediction error of \(f_\theta\) in period \(k\).

The class without time is the subclass
\[
\mathcal F_{\mathrm{nt}}
:=
\left\{
f_\theta\in\mathcal F:
f_\theta(x,t)=f_\theta(x,t')
\text{ for all }x,\ t,\ t'
\right\}.
\]
That is, \(\mathcal F_{\mathrm{nt}}\) consists of predictors that do not depend explicitly on \(T\).

Let
\[
\theta_k^*\in\arg\min_{\theta\in\Theta}R_k(\theta)
\]
be a population risk minimizer in period \(k\).

A model trained on the 2020 data targets \(\theta_0^*\). Its out-of-time performance in a later period \(k>0\) is measured by \(R_k(\theta_0^*)\). Out-of-time deterioration occurs when the predictor that is optimal for 2020 is no longer optimal under the later-period distribution \(P_k\), so that
\[
R_k(\theta_0^*)-R_k(\theta_k^*)
\]
is large.

This is especially clear in the linear case. Suppose
\[
f_\theta(x,t)=\theta^\top \phi(x,t),
\]
where \(\phi(x,t)\) is a \(p\)-dimensional feature vector. If needed, the intercept is absorbed into \(\phi(x,t)\).

Define
\[
A_k := E_k\!\left[\phi(X,T)\phi(X,T)^\top\right],
\qquad
b_k := E_k\!\left[\phi(X,T)Y\right].
\]
Then
\begin{align*}
R_k(\theta)
&= E_k\!\left[\bigl(Y-\theta^\top\phi(X,T)\bigr)^2\right] \\
&= E_k[Y^2] - 2\theta^\top E_k[\phi(X,T)Y] + \theta^\top E_k[\phi(X,T)\phi(X,T)^\top]\theta \\
&= E_k[Y^2] - 2\theta^\top b_k + \theta^\top A_k \theta.
\end{align*}

If \(A_k\) is positive definite, then the minimizer is unique and is given by
\[
\theta_k^* = A_k^{-1} b_k.
\]
Moreover,
\begin{align*}
R_k(\theta) - R_k(\theta_k^*)
&=
\theta^\top A_k \theta - 2\theta^\top b_k + 2(\theta_k^*)^\top b_k - (\theta_k^*)^\top A_k \theta_k^* \\
&=
\theta^\top A_k \theta - 2\theta^\top A_k\theta_k^* + (\theta_k^*)^\top A_k \theta_k^* \\
&=
(\theta-\theta_k^*)^\top A_k (\theta-\theta_k^*).
\end{align*}
In particular,
\[
R_k(\theta_0^*) - R_k(\theta_k^*)
=
(\theta_0^*-\theta_k^*)^\top A_k (\theta_0^*-\theta_k^*).
\]

This identity shows that, for linear predictors, the excess out-of-time risk is exactly the discrepancy between the 2020-optimal parameter \(\theta_0^*\) and the period-\(k\)-optimal parameter \(\theta_k^*\), measured in the geometry induced by the later-period second-moment matrix \(A_k\).

The key point is therefore this. Even though the retrieval map \(f^*\) itself is fixed, the best approximation to \(f^*\) within a restricted predictor class may still change over time, because the input distribution changes over time.

\subsection{A linear model without time}

We now specialize the general discussion in Section~A.1 to a linear predictor
that uses only the original emulator features and does not include time.

Let \(\phi(x)\in\mathbb{R}^p\) denote the feature vector excluding time. For
\(\theta\in\mathbb{R}^p\), write
\[
f_\theta(x)=\theta^\top\phi(x),
\qquad
\mathcal F_{\rm lin}
=
\{f_\theta:\theta\in\mathbb{R}^p\}.
\]
In this subsection,
\[
R_0(\theta)
=
E_0\left[
\left\{Y-\theta^\top\phi(X)\right\}^2
\right].
\]
Let
\[
\theta_0^* \in \arg\min_{\theta\in\mathbb{R}^p} R_0(\theta)
\]
be a period-0 risk minimizer, and write
\[
f_0(X):=(\theta_0^*)^\top \phi(X).
\]
Then \(f_0\) is a population-optimal linear predictor in period 0 within
\(\mathcal F_{\rm lin}\).

Define the residual
\[
U:=Y-f_0(X).
\]
Since \(Y=f^*(X)\), the residual \(U=Y-f_0(X)\) is a function of \(X\) alone once \(f_0\) is fixed.
Thus \(U\) itself does not depend on \(k\). Moreover, since \(\theta_0^*\) minimizes \(R_0\), it satisfies the normal equation
\[
E_0\!\left[\phi(X)\bigl(Y-(\theta_0^*)^\top\phi(X)\bigr)\right]=0,
\]
that is,
\[
E_0[\phi(X)U]=0.
\]
However, \(E_k[U\mid T]\) may depend on \(k\), because the conditional law of \(X\) given \(T\)
may differ across periods.

Accordingly, for each period \(k\), define
\[
m_k(T):=E_k[U\mid T],\qquad \xi_k:=U-m_k(T),
\]
where the conditional expectation is taken under \(P_k\). Then
\[
E_k[\xi_k \mid T] = 0, \qquad U = m_k(T) + \xi_k.
\]

\begin{proposition}\label{prop:no_time_decomposition}
For each period $k$,
\[
E_k\!\left[\bigl(Y-f_0(X)\bigr)^2\right]
=
E_k\!\left[m_k(T)^2\right]
+
E_k[\xi_k^2].
\]
\end{proposition}

\begin{proof}
Since
\[
Y-f_0(X)=U=m_k(T)+\xi_k,
\]
we have
\[
E_k\!\left[\bigl(Y-f_0(X)\bigr)^2\right]
=
E_k\!\left[\bigl(m_k(T)+\xi_k\bigr)^2\right].
\]
Expanding the square gives
\[
E_k[m_k(T)^2]
+
2E_k[m_k(T)\xi_k]
+
E_k[\xi_k^2].
\]
The cross term vanishes because
\[
E_k[m_k(T)\xi_k]
=
E_k\!\left[m_k(T)\,E_k[\xi_k\mid T]\right]
=
0.
\]
This proves the claim.
\end{proof}

Proposition~\ref{prop:no_time_decomposition} shows that the later-period error of the period-0 linear predictor splits into two parts: the systematic time-dependent component $m_k(T)$, and the remaining variation $\xi_k$. Thus, out-of-time deterioration arises when the residual left by the period-0 predictor has a substantial time-dependent mean.

\subsection{Adding a linear time term}

We now enlarge the linear predictor class by adding a linear time feature.

Let
\[
\widetilde T := T - E_0[T]
\]
be the centered time variable. We consider the augmented class
\[
\widetilde{\mathcal F}_{\mathrm{lin}}
:=
\left\{
g_{\theta,\beta} : g_{\theta,\beta}(X,T)=\theta^\top\phi(X)+\beta \widetilde T,\
\theta\in\mathbb{R}^p,\ \beta\in\mathbb{R}
\right\}.
\]
Its period-0 population risk is
\[
\widetilde R_0(\theta,\beta)
:=
E_0\!\left[\bigl(Y-\theta^\top\phi(X)-\beta \widetilde T\bigr)^2\right].
\]

For simplicity, we assume that the centered time feature is orthogonal, under \(P_0\), to the original feature map:
\[
E_0[\phi(X)\widetilde T]=0.
\]
This is a simplifying assumption used only for interpretation; it is not asserted as a realistic property of the present data. Under this assumption, the added linear time term does not alter the period-0 projection onto the original feature space; it only captures the residual component aligned with time. If the orthogonality condition fails, the coefficient vector on the original features will in general also change after time is added.

Let
\[
(\widetilde\theta_0^*,\beta_0^*)
\in
\arg\min_{\theta\in\mathbb{R}^p,\ \beta\in\mathbb{R}}
\widetilde R_0(\theta,\beta)
\]
be a period-0 minimizer in the augmented class.

\begin{proposition}\label{prop:time_augmented_decomposition}
Assume that
\[
\Sigma_0:=E_0[\phi(X)\phi(X)^\top]
\]
is positive definite,
\[
E_0[\widetilde T^2]>0,
\]
and
\[
E_0[\phi(X)\widetilde T]=0.
\]
Then
\begin{align*}
\widetilde\theta_0^* &= \theta_0^*, \\
\beta_0^*
&=
\frac{E_0[Y\widetilde T]}{E_0[\widetilde T^2]}
=
\frac{E_0\!\left[\{Y-f_0(X)\}\widetilde T\right]}
     {E_0[\widetilde T^2]} .
\end{align*}
Moreover, for each period $k$,
\[
E_k\!\left[\bigl(Y-g_{\widetilde\theta_0^*,\beta_0^*}(X,T)\bigr)^2\right]
=
E_k\!\left[\bigl(m_k(T)-\beta_0^*\widetilde T\bigr)^2\right]
+
E_k[\xi_k^2].
\]
Consequently, if
\[
E_k\!\left[\bigl(m_k(T)-\beta_0^*\widetilde T\bigr)^2\right]
<
E_k\!\left[m_k(T)^2\right],
\]
then the period-0 time-augmented predictor
\(g_{\widetilde\theta_0^*,\beta_0^*}\) has smaller period-\(k\) risk
than the period-0 predictor without time, \(f_0\).
This comparison is between two predictors both fitted under \(P_0\). It does
not assert that \(g_{\widetilde\theta_0^*,\beta_0^*}\) has smaller risk than a
predictor that is re-optimized within the class without time under \(P_k\).
\end{proposition}

\begin{proof}
First expand the period-0 risk:
\begin{align*}
\widetilde R_0(\theta,\beta)
&=
E_0[Y^2]
-2\theta^\top E_0[\phi(X)Y]
+\theta^\top E_0[\phi(X)\phi(X)^\top]\theta \\
&\quad
-2\beta E_0[Y\widetilde T]
+2\beta \theta^\top E_0[\phi(X)\widetilde T]
+\beta^2 E_0[\widetilde T^2].
\end{align*}
By the orthogonality assumption $E_0[\phi(X)\widetilde T]=0$, this becomes
\[
\widetilde R_0(\theta,\beta)
=
R_0(\theta)
-2\beta E_0[Y\widetilde T]
+\beta^2 E_0[\widetilde T^2].
\]
Thus the minimization over $(\theta,\beta)$ separates. Since $\Sigma_0$ is positive definite, the minimizer in $\theta$ is unique and equals $\theta_0^*$. Also, since $E_0[\widetilde T^2]>0$, the minimizer in $\beta$ is unique and given by
\[
\beta_0^*
=
\frac{E_0[Y\widetilde T]}{E_0[\widetilde T^2]}.
\]
Because $f_0(X)=(\theta_0^*)^\top\phi(X)$ and $E_0[\phi(X)\widetilde T]=0$,
\[
E_0[f_0(X)\widetilde T]
=
(\theta_0^*)^\top E_0[\phi(X)\widetilde T]
=
0.
\]
Hence
\[
E_0[Y\widetilde T]
=
E_0\!\left[\bigl(Y-f_0(X)\bigr)\widetilde T\right],
\]
which proves the formula for $\beta_0^*$.

Now fix $k$. Since $\widetilde\theta_0^*=\theta_0^*$,
\[
Y-g_{\widetilde\theta_0^*,\beta_0^*}(X,T)
=
Y-f_0(X)-\beta_0^*\widetilde T
=
U-\beta_0^*\widetilde T.
\]
Using $U=m_k(T)+\xi_k$, we get
\[
Y-g_{\widetilde\theta_0^*,\beta_0^*}(X,T)
=
\bigl(m_k(T)-\beta_0^*\widetilde T\bigr)+\xi_k.
\]
Therefore,
\begin{align*}
E_k\!\left[\bigl(Y-g_{\widetilde\theta_0^*,\beta_0^*}(X,T)\bigr)^2\right]
&=
E_k\!\left[\bigl(m_k(T)-\beta_0^*\widetilde T\bigr)^2\right] \\
&\quad
+2E_k\!\left[\bigl(m_k(T)-\beta_0^*\widetilde T\bigr)\xi_k\right]
+E_k[\xi_k^2].
\end{align*}
The cross term vanishes because $m_k(T)-\beta_0^*\widetilde T$ is measurable with respect to $T$ and $E_k[\xi_k\mid T]=0$. Hence
\[
E_k\!\left[\bigl(Y-g_{\widetilde\theta_0^*,\beta_0^*}(X,T)\bigr)^2\right]
=
E_k\!\left[\bigl(m_k(T)-\beta_0^*\widetilde T\bigr)^2\right]
+
E_k[\xi_k^2].
\]
The final comparison follows by subtracting the identity in Proposition~\ref{prop:no_time_decomposition}.
\end{proof}

The condition in Proposition~\ref{prop:time_augmented_decomposition} has a
simple meaning. It says that the fitted linear time correction
\(\beta_0^*\widetilde T\) is closer to the later-period mean residual component
\(m_k(T)\) than the zero correction is, in \(L_2(P_k)\). This holds, for
example, when \(m_k(T)\) is approximately linear in \(\widetilde T\) with slope
close to \(\beta_0^*\).

Thus, Proposition~\ref{prop:time_augmented_decomposition} gives a simple
population-level explanation of why adding time can improve out-of-time
prediction. Under the orthogonality assumption, the time feature does not alter
the period-0 linear fit in the original features; it only adds a linear
correction for the residual component aligned with time. If the orthogonality
assumption does not hold, this clean separation generally breaks down: adding
time also changes the coefficients on the original features, so the result
should be interpreted as a joint re-projection rather than as a pure residual
correction.

\section{Configurations of Machine-Learning Models}
\label{app:ML-models}

The processed 2020 dataset contains 87013 land observations and 38673 ocean observations. Separate models were trained for each target gas (XCO$_2$ and XCH$_4$), for each surface type (land and ocean), and with and without the normalized time feature.

For each run, the 2020 data were randomly split into a training subset (80\%) and an in-period test subset (20\%). Hyperparameter selection and model fitting were based only on the training subset. The held-out 2020 test subset was used only for in-period evaluation. The 2021--2023 data were used only for out-of-time evaluation.

Each input variable listed in \autoref{tab:feature_descriptions} was standardized using the mean and standard deviation computed from the corresponding 2020 training subset. The same standardization parameters were then applied to the held-out 2020 test subset and to the 2021--2023 out-of-time data. Thus, no information from the evaluation data was used in preprocessing or model fitting.

The following sections describe, for each learning method, the training procedure, hyperparameter selection, and final model configuration.

\subsection{Lasso Regression}
\label{app:Lasso}

The Lasso regularization parameter $\lambda$ \cite{tibshirani1996regression} was selected separately for each target gas and surface type using only the corresponding 2020 training subset. The search was carried out in two stages.

In the first stage, 199 candidate values of $\lambda$ were sampled on a logarithmic scale over $[10^{-6},\,10^{2}]$, and $\lambda=0$ was added, giving 200 candidates in total. Each candidate was evaluated by three-fold cross-validation using the \textit{scikit-learn} implementation \cite{pedregosa2011scikit}, and the corresponding cross-validated mean squared error (CV-MSE) was recorded.

In all cases, the CV-MSE was minimized at values of $\lambda$ very close to zero, such as $\lambda=10^{-6}$, and differed only negligibly from the CV-MSE at slightly larger values such as $\lambda=10^{-3}$.

In the second stage, Bayesian optimization based on Gaussian processes, implemented in \textit{scikit-optimize}, was carried out over the continuous interval $\lambda \in [0,\,15]$. For all target-gas and surface-type combinations, this procedure again selected values of $\lambda$ close to zero. This confirmed that little regularization was needed to minimize the cross-validated prediction error.

In the final models, however, we used slightly larger values of $\lambda$ in order to obtain sparser solutions and hence greater interpretability, while keeping predictive performance essentially unchanged. The resulting choices of $\lambda$ differed by target gas and surface type.

\paragraph{Land.}
For land observations, we used $\lambda = 1\times10^{-3}$ for XCO$_2$ in both the model without time and the time-augmented model. For XCH$_4$, we used $\lambda = 5\times10^{-5}$ without time and $\lambda = 1\times10^{-5}$ with time.

Without the time feature, the Lasso was fitted with 4,323 candidate features and retained 506 nonzero coefficients for XCO$_2$ (about 11.7\%) and 821 for XCH$_4$ (about 19.0\%). With the time feature, the number of candidate features increased to 4,324. The number of retained coefficients then changed only slightly for XCO$_2$, from 506 to 508, but decreased markedly for XCH$_4$, from 821 to 455 (about 10.5\%).

Thus, for land-based XCH$_4$ prediction, adding time led to a substantially sparser solution. This suggests that part of the variation that would otherwise be represented by a larger set of features can instead be captured by the time predictor.

\paragraph{Ocean.}
For ocean observations, we used $\lambda = 1\times10^{-3}$ for XCO$_2$ and $\lambda = 1\times10^{-5}$ for XCH$_4$, both without and with the time feature.

Without the time feature, the Lasso retained 944 nonzero coefficients for XCO$_2$ (about 21.8\%) and 801 for XCH$_4$ (about 18.5\%). With the time feature, the corresponding numbers were 943 for XCO$_2$ and 800 for XCH$_4$, again about 21.8\% and 18.5\% of the candidate features.

Thus, over ocean, adding time had essentially no effect on the sparsity pattern selected by the Lasso.

For each target gas, surface type, and time specification, the features selected by the corresponding Lasso model were also used to define the reduced input set for the k-NN model; see \autoref{app:knn}.

\paragraph{Lasso Models Retrained for TCCON Validation.}
For the main experiments, the Lasso models were trained on random 80\% subsets of the 2020 data. For the TCCON validation only, they were retrained on the full 2020 data for the corresponding surface category, while keeping the regularization parameter $\lambda$ fixed at the value selected in the main training procedure.

Most TCCON stations are located over land, but the collocated satellite footprints also include observations labelled as ``other'', that is, surfaces not classified as pure land or pure ocean. Because the number of ocean collocations was too small for reliable interpretation, the retrained models for the TCCON validation were constructed only for the land (87013 observations) and other (12624 observations) surface categories.

These models were fitted without a 20\% hold-out split, so that all available 2020 observations in the corresponding surface category could be used for training. The regularization parameter $\lambda$ was not re-optimized. Instead, we reused the values selected in the main training procedure. This kept the TCCON-validation models consistent with the main analysis and avoided retuning $\lambda$ on the full 2020 data. For each target gas and each surface category, we fitted both the model without time and the time-augmented model.

The resulting sparse solutions, based on 4,323 candidate features without time and 4,324 candidate features with time, are summarized in \autoref{tab:tccon_lasso_sparsity}.

\begin{table}[htbp]
\centering
\caption{Lasso regularization and sparsity for TCCON validation.
Regularization parameters and sparsity levels are reported for Lasso models
retrained on the full 2020 data, by surface category, target gas, and model
variant. The last column gives the percentage of nonzero coefficients among all
input coefficients.}
\label{tab:tccon_lasso_sparsity}
\renewcommand{\arraystretch}{1.15}
\setlength{\tabcolsep}{5pt}

\begin{tabular}{l l l c r c}
\toprule
\textbf{Surface} & \textbf{Gas} & \textbf{Variant}
& $\lambda$ & \textbf{Nonzero} & \textbf{Nonzero/total (\%)} \\
\midrule
Land  & XCO$_2$ & Without time & $1.0\times10^{-3}$ &  506 & 11.7 \\
Land  & XCO$_2$ & With time    & $1.0\times10^{-3}$ &  508 & 11.7 \\
Land  & XCH$_4$ & Without time & $5.0\times10^{-5}$ &  821 & 19.0 \\
Land  & XCH$_4$ & With time    & $1.0\times10^{-5}$ &  455 & 10.5 \\
Other & XCO$_2$ & Without time & $1.0\times10^{-3}$ & 1064 & 24.6 \\
Other & XCO$_2$ & With time    & $1.0\times10^{-3}$ & 1056 & 24.4 \\
Other & XCH$_4$ & Without time & $1.0\times10^{-5}$ &  854 & 19.7 \\
Other & XCH$_4$ & With time    & $1.0\times10^{-5}$ &  864 & 20.0 \\
\bottomrule
\end{tabular}
\end{table}

\subsection{Neural-Net Architectures}
\label{app:NN}

We implemented fully connected feedforward neural networks \cite{goodfellow2016deep} using the \textit{Keras} API in \textit{TensorFlow}, in order to capture nonlinear relations among the satellite-derived predictors. For each run, the 2020 data were first split into a training subset (80\%) and an in-period test subset (20\%). The training subset was then split again in a 3:1 ratio, giving 64\% of the full 2020 data for training and 16\% for validation. The validation subset was used only to monitor predictive performance and to implement early stopping. All models were trained with the \textit{Adam} optimizer \cite{kingma2017} and mean squared error loss.

Initial hyperparameter tuning was carried out by randomized search. In this initial search, the candidate values were the learning rate $\eta \in \{10^{-4},\,10^{-3},\,10^{-2},\,0.1\}$, hidden-layer width $d \in \{32,64,128,256\}$, number of hidden layers $L \in \{1,2,3\}$, dropout rate $p \in \{0.1,0.2,0.3,0.4\}$, and batch size $B \in \{16,32,64,128,256\}$.

Because training on these large datasets was computationally expensive, the search was limited to 20 sampled configurations, each evaluated by three-fold cross-validation. This stage was used mainly to identify reasonable starting architectures for the final manual refinement.

The final architectures were then obtained by limited manual refinement of these initial configurations. In particular, for some settings we adjusted the depth and width of the network when this improved validation performance. For land-based XCH$_4$, this led to the use of a four-hidden-layer architecture. Dropout was not retained in the final models, because it did not improve validation performance. In the final specification, early stopping was used as the sole regularization device. Training was run for at most 400 epochs with batch size 32, and the weights attaining the lowest validation loss were retained, with a patience parameter of 50 epochs. Rectified linear unit activations \cite{nair2010rectified} were used in all hidden layers, followed by a single linear output unit.

\paragraph{Land.}
For land-based XCO$_2$, we used a network with three hidden layers of sizes 64, 32, and 16, followed by a single linear output unit. For land-based XCH$_4$, we used a deeper network with four hidden layers of sizes 64, 32, 16, and 8, again followed by a single linear output unit. This deeper architecture was selected during the final manual refinement stage because it gave better validation performance.

\paragraph{Ocean.}
For ocean-based XCO$_2$, we used a network with three hidden layers of sizes 128, 32, and 32, followed by a single linear output unit. For ocean-based XCH$_4$, we used a wider network with three hidden layers of sizes 256, 64, and 16, again followed by a single linear output unit.

Thus, separate neural-network architectures were used for each target gas and surface type. The final architectures, including the corresponding time-augmented versions, are summarized in \autoref{tab:nn_architectures}.

\begin{table}[htbp]
\centering
\renewcommand{\arraystretch}{1.3}
\setlength{\tabcolsep}{4.2pt}
\caption{Final neural-network architectures. Architectures are shown by surface type, target gas, and model variant. All models were trained using the Adam optimizer.}
\label{tab:nn_architectures}
\begin{tabular}{lllccc}
\hline
\textbf{Surface} & \textbf{Gas} & \textbf{Variant} & \textbf{Hidden layers} & \textbf{Layer sizes} & \textbf{Batch size} \\
\hline
Land  & XCO$_2$ & Without time & 3 & 64 $\rightarrow$ 32 $\rightarrow$ 16 & 32 \\
Land  & XCO$_2$ & With time    & 3 & 64 $\rightarrow$ 32 $\rightarrow$ 16 & 32 \\
Land  & XCH$_4$ & Without time & 4 & 64 $\rightarrow$ 32 $\rightarrow$ 16 $\rightarrow$ 8 & 32 \\
Land  & XCH$_4$ & With time    & 4 & 64 $\rightarrow$ 32 $\rightarrow$ 16 $\rightarrow$ 8 & 32 \\
\hline
Ocean & XCO$_2$ & Without time & 3 & 128 $\rightarrow$ 32 $\rightarrow$ 32 & 32 \\
Ocean & XCO$_2$ & With time    & 3 & 128 $\rightarrow$ 32 $\rightarrow$ 32 & 32 \\
Ocean & XCH$_4$ & Without time & 3 & 256 $\rightarrow$ 64 $\rightarrow$ 16 & 32 \\
Ocean & XCH$_4$ & With time    & 3 & 256 $\rightarrow$ 64 $\rightarrow$ 16 & 32 \\
\hline
\end{tabular}
\end{table}

\subsection{$k$-NN Regression}
\label{app:knn}

We implemented $k$-nearest neighbours ($k$-NN) regression, following the approach of Kanagawa~\cite{kanagawa2024fast}. Because $k$-NN is sensitive to the curse of dimensionality, we reduced the input dimension and trained each model using only the 20 features selected by the corresponding Lasso model.

During model selection, the number of neighbours $k$ was varied from 5 to 500 in steps of 3, and four distance metrics were considered: Manhattan, Euclidean, Chebyshev, and cosine. For each combination of $k$ and distance metric, we computed the leave-one-out cross-validation (LOOCV) score. In general, smaller values of $k$ gave better LOOCV performance, and no intermediate value was consistently preferred.

To assess the practical effect of neighbourhood size, we also compared models with $k=3$ and $k=400$ on the out-of-time data from 2021 to 2023. The model with $k=3$ consistently outperformed the model with $k=400$, which supported the LOOCV-based preference for smaller values of $k$. The final configurations, including the corresponding time-augmented variants, are summarized in \autoref{tab:knn_hyperparameters}.

\begin{table}[htbp]
\centering
\renewcommand{\arraystretch}{1.3}
\setlength{\tabcolsep}{4.5pt}
\caption{Final $k$-NN hyperparameter settings. Settings are shown by surface type, target gas, and model variant. For all models, the input space was restricted to the 20 features selected by the corresponding Lasso model.}
\label{tab:knn_hyperparameters}
\begin{tabular}{lllcc}
\hline
\textbf{Surface} & \textbf{Gas} & \textbf{Variant} & \textbf{Neighbours ($k$)} & \textbf{Distance metric} \\
\hline
Land  & XCO$_2$ & Without time &  8 & Manhattan \\
Land  & XCO$_2$ & With time    &  8 & Manhattan \\
Land  & XCH$_4$ & Without time &  3 & Manhattan \\
Land  & XCH$_4$ & With time    &  3 & Manhattan \\
\hline
Ocean & XCO$_2$ & Without time & 17 & Cosine \\
Ocean & XCO$_2$ & With time    & 17 & Cosine \\
Ocean & XCH$_4$ & Without time & 11 & Cosine \\
Ocean & XCH$_4$ & With time    & 11 & Cosine \\
\hline
\end{tabular}
\end{table}

\paragraph{Land.}
For land observations, the Manhattan distance gave the lowest LOOCV error for both target gases. The selected numbers of neighbours were $k=8$ for XCO$_2$ and $k=3$ for XCH$_4$. This preference for small neighbourhoods suggests that accurate prediction over land depends on finding closely similar observations in the reduced feature space.

\paragraph{Ocean.}
For ocean observations, the cosine distance gave the lowest LOOCV error for both target gases. The selected numbers of neighbours were larger than those over land, namely $k=17$ for XCO$_2$ and $k=11$ for XCH$_4$. This suggests that, over ocean, prediction benefits from averaging over a broader neighbourhood in the reduced feature space.

\subsection{XGBoost Regression}
\label{app:xgb}

We implemented gradient-boosted decision trees using the \textit{XGBoost} library \cite{chen2016xgboost}. Unlike the $k$-NN models, the XGBoost models were trained on the full feature set. Hyperparameters were tuned by randomized search, using 150 sampled configurations for the land models and 250 for the ocean models. Each sampled configuration was evaluated by three-fold cross-validation with GPU-accelerated histogram-based trees.

For the XGBoost models, the hyperparameter search spaces were defined separately for CO$_2$ and CH$_4$.

For XCO$_2$, the search space comprised the following hyperparameters:
\begin{itemize}
    \item \texttt{max\_depth} \(\in \{3,6,9,15,20,25,35,45,55,70\}\), which controls the maximum depth of each tree;
    \item \texttt{min\_child\_weight} \(\in \{5,10,15,20,25,30,50,70,90\}\), which controls the minimum summed instance weight required in a child node;
    \item \texttt{subsample} \(\in \{0.1,0.25,0.4,0.55,0.65,0.8,0.9,1\}\), which controls the fraction of training instances sampled for each tree;
    \item \texttt{colsample\_bytree} \(\in \{0.2,0.3,0.45,0.6,0.7,0.8,0.9,1\}\), which controls the fraction of features sampled for each tree;
    \item \texttt{eta} \(\in \{10^{-5},10^{-4},10^{-3},0.05,0.1,0.25\}\), which controls the learning rate;
    \item \texttt{n\_estimators} \(\in \{50,100,150,200,250,300,500,600\}\), which controls the number of boosting rounds;
    \item \texttt{gamma} \(\in \{1,3,7,9,12,15,17,25,40,50,65\}\), which controls the minimum loss reduction required for a further split;
    \item \texttt{alpha} \(\in \{0.1,0.5,2,5,10,13,17,25,35,50\}\), which is the L1 regularization coefficient;
    \item \texttt{lambda} \(\in \{0.2,0.5,0.9,4,6,10,15,20,30,45\}\), which is the L2 regularization coefficient.
\end{itemize}

For XCH$_4$, the search space comprised the following hyperparameters:
\begin{itemize}
    \item \texttt{max\_depth} \(\in \{3,6,9,15,20,25,35,45,55,70\}\), which controls the maximum depth of each tree;
    \item \texttt{min\_child\_weight} \(\in \{5,10,15,20,25,30,50,70,90\}\), which controls the minimum summed instance weight required in a child node;
    \item \texttt{subsample} \(\in \{0.1,0.25,0.35,0.5,0.7\}\), which controls the fraction of training instances sampled for each tree;
    \item \texttt{colsample\_bytree} \(\in \{0.2,0.3,0.45,0.6,0.8\}\), which controls the fraction of features sampled for each tree;
    \item \texttt{eta} \(\in \{10^{-5},10^{-4},10^{-3},0.05,0.1,0.25\}\), which controls the learning rate;
    \item \texttt{n\_estimators} \(\in \{50,100,150,200,250,300,500,600\}\), which controls the number of boosting rounds;
    \item \texttt{gamma} \(\in \{1,3,7,9,12,15,17,25,40,50,65\}\), which controls the minimum loss reduction required for a further split;
    \item \texttt{alpha} \(\in \{0.1,0.5,2,5,10,13,17,25,35,50\}\), which is the L1 regularization coefficient;
    \item \texttt{lambda} \(\in \{0.2,0.5,0.9,4,6,10,15,20,30,45\}\), which is the L2 regularization coefficient.
\end{itemize}

\begin{table}[htbp]
\centering
\footnotesize
\renewcommand{\arraystretch}{1.3}
\setlength{\tabcolsep}{4.2pt}
\caption{Final XGBoost hyperparameter configurations. Configurations are shown by surface type, target gas, and model variant. Rounds, Depth, LR, Sub., Col., and MCW denote \texttt{n\_estimators}, \texttt{max\_depth}, \texttt{eta}, \texttt{subsample}, \texttt{colsample\_bytree}, and \texttt{min\_child\_weight}, respectively. The symbols $\gamma$, $\alpha$, and $\lambda$ denote \texttt{gamma}, \texttt{alpha}, and \texttt{lambda}, respectively.}
\label{tab:xgboost_hyperparameters}
\begin{tabular}{lllccccccccc}
\hline
\textbf{Surface} & \textbf{Gas} & \textbf{Variant} &
\textbf{Rounds} & \textbf{Depth} & \textbf{LR} &
\textbf{Sub.} & \textbf{Col.} & \textbf{MCW} &
$\boldsymbol{\gamma}$ & $\boldsymbol{\alpha}$ & $\boldsymbol{\lambda}$ \\
\hline
Land  & XCO$_2$ & Without time & 600 &  9 & 0.05 & 0.40 & 0.70 & 30 & 1.0 &  5.0 & 20.0 \\
Land  & XCO$_2$ & With time    & 600 &  9 & 0.05 & 0.40 & 0.70 & 30 & 1.0 &  5.0 & 20.0 \\
Land  & XCH$_4$ & Without time & 100 & 20 & 0.05 & 0.35 & 0.45 & 30 & 1.0 &  5.0 &  0.9 \\
Land  & XCH$_4$ & With time    & 100 & 20 & 0.05 & 0.35 & 0.45 & 30 & 1.0 &  5.0 &  0.9 \\
\hline
Ocean & XCO$_2$ & Without time & 600 &  9 & 0.05 & 0.90 & 1.00 & 25 & 3.0 & 17.0 &  0.9 \\
Ocean & XCO$_2$ & With time    & 600 &  9 & 0.05 & 0.90 & 1.00 & 25 & 3.0 & 17.0 &  0.9 \\
Ocean & XCH$_4$ & Without time & 200 & 20 & 0.25 & 0.70 & 0.60 & 30 & 1.0 &  0.5 & 15.0 \\
Ocean & XCH$_4$ & With time    & 200 & 20 & 0.25 & 0.70 & 0.60 & 30 & 1.0 &  0.5 & 15.0 \\
\hline
\end{tabular}
\end{table}

\paragraph{Land.}
For land-based XCO$_2$, the selected model used trees of depth 9, 600 boosting rounds, and relatively strong L2 regularization ($\lambda = 20$). For land-based XCH$_4$, the selected model used deeper trees (maximum depth 20), stronger stochastic subsampling (\texttt{subsample} $= 0.35$), and 100 boosting rounds. Thus, the two land models relied on different mechanisms for controlling model complexity.

\paragraph{Ocean.}
For ocean-based XCO$_2$, the selected model used near-full sampling (\texttt{subsample} $= 0.90$, \texttt{colsample\_bytree} $= 1.00$), 600 boosting rounds, and relatively strong L1 regularization ($\alpha = 17$). For ocean-based XCH$_4$, the selected model used deeper trees, a higher learning rate ($\eta = 0.25$), and 200 boosting rounds. Thus, the two ocean models also required different hyperparameter regimes.

\section{Additional Tables}
\label{app:additional-tables}

This appendix collects supplementary tables referenced in the main text. Table~\ref{tab:tccon_sites} lists the TCCON sites used in the external validation. Tables~\ref{tab:co2_lasso_time_top40_land_ocean} and \ref{tab:ch4_lasso_time_top40_land_ocean} report the top 40 features selected by the time-augmented Lasso models for XCO$_2$ and XCH$_4$, respectively, over land and ocean. Tables~\ref{tab:OOT_land_xco2}--\ref{tab:OOT_ocean_xch4} give the full NRMSE results by gas, surface type, and year.

\begin{table}[htbp]
\centering
\caption{TCCON validation sites. Sites are listed with geographic coordinates and dataset references.}
\label{tab:tccon_sites}
\renewcommand{\arraystretch}{1.1}
\setlength{\tabcolsep}{5pt}
\begin{tabular}{l c c l}
\hline
\textbf{Site} & \textbf{Latitude} & \textbf{Longitude} & \textbf{Reference} \\
\hline
Bremen              & $53.10^\circ$ N  & $8.85^\circ$ E    & \cite{notholt_petri_warneke_buschmann_2022} \\
Burgos              & $18.533^\circ$ N & $120.650^\circ$ E & \cite{morino_velazco_hori_uchino_griffith_2022} \\
Caltech (Pasadena)  & $34.136^\circ$ N & $118.127^\circ$ W & \cite{wennberg_roehl_wunch_blavier_toon_allen_treffers_laughner_2022} \\
East Trout Lake     & $54.354^\circ$ N & $104.987^\circ$ W & \cite{wunch_mendonca_colebatch_allen_blavier_kunz_roche_hedelius_neufeld_springett_worthy_kessler_strong_2022} \\
Four Corners        & $36.707^\circ$ N & $108.480^\circ$ W & \cite{dubey_lindenmaier_henderson_allen_roehl_blavier_love_wunch_2022} \\
Indianapolis        & $39.861^\circ$ N & $86.004^\circ$ W  & \cite{iraci_podolske_hillyard_roehl_wennberg_blavier_landeros_allen_wunch_zavaleta_quigley_osterman_barrow_barney_2022} \\
JPL02               & $34.202^\circ$ N & $118.175^\circ$ W & \cite{wennberg_roehl_blavier_wunch_allen_2022} \\
Karlsruhe           & $49.100^\circ$ N & $8.439^\circ$ E   & \cite{hase_herkommer_gross_blumenstock_kiel_dohe_2024} \\
Lauder01            & $45.038^\circ$ S & $169.684^\circ$ E & \cite{sherlock_connor_robinson_shiona_smale_pollard_2022_lauder1} \\
Lauder02            & $45.038^\circ$ S & $169.684^\circ$ E & \cite{sherlock_connor_robinson_shiona_smale_pollard_2022_lauder2} \\
Lauder03            & $45.038^\circ$ S & $169.684^\circ$ E & \cite{pollard_robinson_shiona_2022} \\
Lamont              & $36.604^\circ$ N & $97.486^\circ$ W  & \cite{wennberg_wunch_roehl_blavier_toon_allen_2022} \\
Manaus              & $3.213^\circ$ S  & $60.598^\circ$ W  & \cite{dubey_henderson_allen_blavier_roehl_wunch_2022} \\
Nicosia             & $35.141^\circ$ N & $33.381^\circ$ E  & \cite{petri_vrekoussis_rousogenous_warneke_sciare_notholt_2024} \\
Orléans             & $47.970^\circ$ N & $2.113^\circ$ E   & \cite{warneke_petri_notholt_buschmann_2024} \\
Paris               & $48.846^\circ$ N & $2.356^\circ$ E   & \cite{te_jeseck_janssen_2022} \\
Park Falls          & $45.945^\circ$ N & $90.273^\circ$ W  & \cite{wennberg_roehl_wunch_toon_blavier_washenfelder_keppel-aleks_allen_2022} \\
Réunion Island      & $20.901^\circ$ S & $55.485^\circ$ E  & \cite{demaziere_sha_desmet_hermans_scolas_kumps_zhou_metzger_duflot_cammas_2022} \\
Rikubetsu           & $43.457^\circ$ N & $143.766^\circ$ E & \cite{morino_ohyama_hori_ikegami_2022_Rikubetsu} \\
Saga                & $33.241^\circ$ N & $130.288^\circ$ E & \cite{shiomi_kawakami_ohyama_arai_okumura_ikegami_usami_2022} \\
Sodankylä           & $67.367^\circ$ N & $26.631^\circ$ E  & \cite{kivi_heikkinen_kyro_2022} \\
Tsukuba             & $36.051^\circ$ N & $140.122^\circ$ E & \cite{morino_ohyama_hori_ikegami_2022_tsukuba} \\
Xianghe             & $39.750^\circ$ N & $116.960^\circ$ E & \cite{zhou_wang_kumps_hermans_nan_2022} \\
\hline
\end{tabular}
\end{table}

\begin{table}[htbp]
\caption{Top XCO$_2$ Lasso features. The table reports the top 40 features for the time-augmented Lasso model over land and ocean. Wavenumbers are given in cm$^{-1}$; non-spectral features are indicated by ``--''.}
\centering
\renewcommand{\arraystretch}{1.1}
\setlength{\tabcolsep}{4pt}
\begin{tabular}{cccc@{\hspace{1.5em}}cccc}
\toprule
\multicolumn{4}{c}{Land} & \multicolumn{4}{c}{Ocean} \\
\cmidrule(lr){1-4} \cmidrule(lr){5-8}
Rank & Variable & Coefficient & cm$^{-1}$ & Rank & Variable & Coefficient & cm$^{-1}$ \\
\midrule
1  & \texttt{xco2\_ap} &  1.1251 & --       & 1  & \texttt{ch1254} &  1.727903 & 6179.47  \\
2  & \texttt{ch3760}   &  1.0996 & 4799.61  & 2  & \texttt{ch1069} & -1.227399 & 13162.46 \\
3  & \texttt{ch3759}   &  0.9878 & 5198.60  & 3  & \texttt{SZs}    &  1.139208 & --       \\
4  & \texttt{ch314}    &  0.8987 & 13011.72 & 4  & \texttt{ch305}  &  1.021829 & 13009.93 \\
5  & \texttt{ch1254}   & -0.8889 & 6179.47  & 5  & \texttt{ch2062} & -0.942576 & 6340.63  \\
6  & \texttt{ch193}    &  0.7952 & 12987.59 & 6  & \texttt{ch1399} & -0.851444 & 6208.39  \\
7  & \texttt{ch1081}   & -0.7828 & 13164.85 & 7  & \texttt{ch1081} & -0.840253 & 13164.85 \\
8  & \texttt{ch304}    &  0.6616 & 13009.73 & 8  & \texttt{ch314}  &  0.798474 & 13011.72 \\
9  & \texttt{ch316}    & -0.6604 & 13012.12 & 9  & \texttt{ch1548} & -0.797852 & 6238.11  \\
10 & \texttt{ch2256}   & -0.6576 & 5899.26  & 10 & \texttt{ch3761} & -0.797470 & 4799.80  \\
11 & \texttt{ch3515}   & -0.6558 & 5150.65  & 11 & \texttt{ch2045} & -0.796562 & 6337.24  \\
12 & \texttt{ch3764}   &  0.5881 & 4800.40  & 12 & \texttt{ch1527} & -0.743983 & 6233.92  \\
13 & \texttt{ch589}    &  0.5427 & 13066.71 & 13 & \texttt{ch2142} & -0.741416 & 6356.59  \\
14 & \texttt{ch3589}   &  0.5190 & 5165.41  & 14 & \texttt{ch2054} & -0.730762 & 6339.03  \\
15 & \texttt{ch514}    & -0.5170 & 13051.75 & 15 & \texttt{ch1550} &  0.687708 & 6238.51  \\
16 & \texttt{ch501}    &  0.5114 & 13049.15 & 16 & \texttt{ch1402} &  0.636261 & 6208.99  \\
17 & \texttt{ch251}    &  0.5063 & 12999.16 & 17 & \texttt{ch1080} & -0.635616 & 13164.65 \\
18 & \texttt{ch3566}   & -0.4775 & 5160.82  & 18 & \texttt{ch1543} &  0.624887 & 6237.11  \\
19 & \texttt{ch1074}   & -0.4744 & 13163.46 & 19 & \texttt{ch2128} & -0.598341 & 6353.79  \\
20 & \texttt{ch76}     &  0.4655 & 12964.25 & 20 & \texttt{ch2064} &  0.582891 & 6341.03  \\
21 & \texttt{co2\_5}   &  0.4628 & --       & 21 & \texttt{ch608}  & -0.580630 & 13070.50 \\
22 & \texttt{ch1076}   & -0.4511 & 13163.86 & 22 & \texttt{ch46}   & -0.572837 & 12958.26 \\
23 & \texttt{ch358}    &  0.4449 & 13020.50 & 23 & \texttt{ch4165} & -0.562910 & 4880.38  \\
24 & \texttt{ch1080}   & -0.4319 & 13164.65 & 24 & \texttt{ch3903} & -0.562642 & 4828.12  \\
25 & \texttt{ch2996}   &  0.4301 & 6047.03  & 25 & \texttt{ch2144} &  0.561964 & 6356.99  \\
26 & \texttt{ch261}    & -0.4273 & 13001.15 & 26 & \texttt{ch1444} & -0.547034 & 6217.36  \\
27 & \texttt{ch1563}   &  0.4259 & 6241.10  & 27 & \texttt{ch2000} & -0.526930 & 6328.26  \\
28 & \texttt{T16}      & -0.4223 & --       & 28 & \texttt{ch852}  & -0.523511 & 13119.17 \\
29 & \texttt{ch546}    &  0.3984 & 13058.13 & 29 & \texttt{ch251}  &  0.506768 & 12999.16 \\
30 & \texttt{ch679}    &  0.3937 & 13084.66 & 30 & \texttt{ch731}  & -0.502339 & 13095.04 \\
31 & \texttt{ch819}    & -0.3906 & 13112.59 & 31 & \texttt{ch1541} & -0.496560 & 6236.71  \\
32 & \texttt{SZs}      & -0.3831 & --       & 32 & \texttt{ch1453} & -0.486578 & 6219.16  \\
33 & \texttt{co2\_4}   & -0.3701 & --       & 33 & \texttt{ch3824} & -0.482009 & 4812.37  \\
34 & \texttt{co2\_9}   &  0.3695 & --       & 34 & \texttt{ch366}  &  0.475613 & 13022.10 \\
35 & \texttt{ch3817}   &  0.3694 & 4810.97  & 35 & \texttt{ch2264} &  0.474914 & 5901.02  \\
36 & \texttt{ch2597}   & -0.3650 & 5967.44  & 36 & \texttt{ch194}  &  0.465847 & 12987.79 \\
37 & \texttt{ch1078}   & -0.3534 & 13164.26 & 37 & \texttt{ch4104} & -0.461703 & 4868.22  \\
38 & \texttt{t}        &  0.3526 & --       & 38 & \texttt{ch4220} & -0.456703 & 4891.35  \\
39 & \texttt{ch3517}   & -0.3525 & 5151.05  & 39 & \texttt{lat}    &  0.445869 & --       \\
40 & \texttt{ch257}    & -0.3497 & 13000.35 & 40 & \texttt{ch1561} & -0.444901 & 6240.70  \\
\bottomrule
\end{tabular}
\label{tab:co2_lasso_time_top40_land_ocean}
\end{table}

\begin{table}[htbp]
\caption{Top XCH$_4$ Lasso features. The table reports the top 40 features for the time-augmented Lasso model over land and ocean. Wavenumbers are given in cm$^{-1}$; non-spectral features are indicated by ``--''.}
\centering
\renewcommand{\arraystretch}{1.1}
\setlength{\tabcolsep}{4pt}
\begin{tabular}{cccc@{\hspace{1.5em}}cccc}
\toprule
\multicolumn{4}{c}{Land} & \multicolumn{4}{c}{Ocean} \\
\cmidrule(lr){1-4} \cmidrule(lr){5-8}
Rank & Variable & Coefficient & cm$^{-1}$ & Rank & Variable & Coefficient & cm$^{-1}$ \\
\midrule
1  & \texttt{ch3143}        & -0.012718 & 6076.35 & 1  & \texttt{ch3143}   & -0.024172 & 6076.35 \\
2  & \texttt{ch3191}        & -0.009218 & 6085.92 & 2  & \texttt{ch3093}   & -0.020122 & 6066.38 \\
3  & \texttt{xch4\_ap}      &  0.009119 & --      & 3  & \texttt{ch2504}   & -0.019994 & 5948.89 \\
4  & \texttt{ch3817}        &  0.008848 & 4810.97 & 4  & \texttt{ch3191}   & -0.019505 & 6085.92 \\
5  & \texttt{ch3093}        & -0.007855 & 6066.38 & 5  & \texttt{ch2765}   & -0.017586 & 6000.95 \\
6  & \texttt{dop\_v\_earth} & -0.007674 & --      & 6  & \texttt{ch2617}   & -0.016836 & 5971.43 \\
7  & \texttt{SAc}           & -0.007113 & --      & 7  & \texttt{ch2561}   & -0.016643 & 5960.26 \\
8  & \texttt{ch2330}        & -0.006993 & 5914.19 & 8  & \texttt{ch3193}   &  0.016369 & 6086.32 \\
9  & \texttt{ch4\_1}        &  0.006878 & --      & 9  & \texttt{ch2776}   & -0.015705 & 6003.15 \\
10 & \texttt{ch3242}        &  0.006875 & 6096.09 & 10 & \texttt{ch3240}   & -0.014668 & 6095.70 \\
11 & \texttt{ch3240}        & -0.005774 & 6095.70 & 11 & \texttt{ch3043}   & -0.014150 & 6056.40 \\
12 & \texttt{ch2389}        & -0.005677 & 5925.96 & 12 & \texttt{ch2389}   & -0.013831 & 5925.96 \\
13 & \texttt{ch2504}        & -0.005616 & 5948.89 & 13 & \texttt{ch3095}   &  0.012806 & 6066.77 \\
14 & \texttt{ch2760}        & -0.005548 & 5999.95 & 14 & \texttt{ch3192}   & -0.012570 & 6086.12 \\
15 & \texttt{ch3144}        &  0.005423 & 6076.55 & 15 & \texttt{ch2782}   &  0.012369 & 6004.34 \\
16 & \texttt{ch3043}        & -0.005290 & 6056.40 & 16 & \texttt{ch3242}   &  0.012280 & 6096.09 \\
17 & \texttt{ch3289}        & -0.005092 & 6105.47 & 17 & \texttt{ch2767}   &  0.012260 & 6001.35 \\
18 & \texttt{ch4012}        &  0.004931 & 4849.87 & 18 & \texttt{ch3144}   &  0.011167 & 6076.55 \\
19 & \texttt{ch2782}        &  0.004858 & 6004.34 & 19 & \texttt{ch3289}   & -0.011018 & 6105.47 \\
20 & \texttt{ch3145}        &  0.004532 & 6076.75 & 20 & \texttt{ch2760}   & -0.009698 & 5999.95 \\
21 & \texttt{ch2776}        & -0.004460 & 6003.15 & 21 & \texttt{ch2562}   &  0.009367 & 5960.46 \\
22 & \texttt{SAs}           &  0.004324 & --      & 22 & \texttt{ch2773}   & -0.008751 & 6002.55 \\
23 & \texttt{ch4032}        &  0.004168 & 4853.86 & 23 & \texttt{ch2762}   &  0.008566 & 6000.35 \\
24 & \texttt{ch2270}        & -0.004067 & 5902.22 & 24 & \texttt{xch4\_ap} &  0.008565 & --      \\
25 & \texttt{t}             &  0.003925 & --      & 25 & \texttt{ch3291}   &  0.008224 & 6105.87 \\
26 & \texttt{T1}            & -0.003809 & --      & 26 & \texttt{ch3145}   &  0.008180 & 6076.75 \\
27 & \texttt{ch2272}        & -0.003787 & 5902.62 & 27 & \texttt{ch2506}   &  0.007829 & 5949.29 \\
28 & \texttt{ch3194}        &  0.003739 & 6086.52 & 28 & \texttt{ch1080}   & -0.007669 & 13164.65 \\
29 & \texttt{ch3195}        &  0.003717 & 6086.72 & 29 & \texttt{ch1081}   & -0.007289 & 13164.85 \\
30 & \texttt{ch3149}        &  0.003540 & 6077.54 & 30 & \texttt{ch3239}   & -0.007045 & 6095.50 \\
31 & \texttt{ch3095}        &  0.003482 & 6066.77 & 31 & \texttt{ch2505}   &  0.006958 & 5949.09 \\
32 & \texttt{ch4015}        & -0.003438 & 4850.46 & 32 & \texttt{ch2272}   & -0.006876 & 5902.62 \\
33 & \texttt{T17}           & -0.003431 & --      & 33 & \texttt{ch2777}   &  0.006430 & 6003.35 \\
34 & \texttt{T12}           &  0.003291 & --      & 34 & \texttt{ch2619}   &  0.006349 & 5971.83 \\
35 & \texttt{ch2617}        & -0.003133 & 5971.43 & 35 & \texttt{ch2618}   &  0.006278 & 5971.63 \\
36 & \texttt{ch1253}        & -0.003082 & 13163.37 & 36 & \texttt{ch2390}   &  0.006264 & 5926.15 \\
37 & \texttt{ch4184}        & -0.003079 & 4884.17 & 37 & \texttt{ch4\_1}   &  0.006233 & --      \\
38 & \texttt{ch3824}        & -0.003073 & 4812.37 & 38 & \texttt{ch2771}   &  0.005806 & 6002.15 \\
39 & \texttt{ch4047}        &  0.003058 & 4856.85 & 39 & \texttt{ch2992}   & -0.005741 & 6046.23 \\
40 & \texttt{ch4\_7}        &  0.003057 & --      & 40 & \texttt{ch2770}   & -0.005549 & 6001.95 \\
\bottomrule
\end{tabular}
\label{tab:ch4_lasso_time_top40_land_ocean}
\end{table}

\begin{table}[htbp]
\centering
\caption{Land XCO$_2$ out-of-time NRMSE. NRMSE is reported for the 2020 in-period test set and the 2021--2023 out-of-time test sets. Entries are means $\pm$ standard deviations over 10 runs.}
\label{tab:OOT_land_xco2}
\footnotesize
\renewcommand{\arraystretch}{1.2}
\setlength{\tabcolsep}{6pt}
\begin{tabular}{lcccc}
\toprule
\textbf{Model} & \textbf{2020} & \textbf{2021} & \textbf{2022} & \textbf{2023} \\
\midrule
k-NN
& 0.00488 $\pm$ 0.00004
& 0.00634 $\pm$ 0.00001
& 0.00828 $\pm$ 0.00002
& 0.01123 $\pm$ 0.00004 \\
k-NN (time)
& 0.00456 $\pm$ 0.00002
& 0.00692 $\pm$ 0.00001
& 0.00895 $\pm$ 0.00003
& 0.01191 $\pm$ 0.00006 \\
XGBoost
& 0.00440 $\pm$ 0.00005
& 0.00544 $\pm$ 0.00002
& 0.00747 $\pm$ 0.00009
& 0.01090 $\pm$ 0.00016 \\
XGBoost (time)
& \textbf{0.00439} $\pm$ 0.00004
& 0.00544 $\pm$ 0.00003
& 0.00744 $\pm$ 0.00008
& 0.01088 $\pm$ 0.00012 \\
NN
& 0.00456 $\pm$ 0.00003
& \textbf{0.00497} $\pm$ 0.00009
& 0.00546 $\pm$ 0.00017
& 0.00596 $\pm$ 0.00043 \\
NN (time)
& 0.00458 $\pm$ 0.00003
& 0.01572 $\pm$ 0.00116
& 0.01731 $\pm$ 0.00188
& 0.01869 $\pm$ 0.00307 \\
Lasso
& 0.00495 $\pm$ 0.00004
& 0.00511 $\pm$ 0.00001
& \textbf{0.00525} $\pm$ 0.00001
& \textbf{0.00534} $\pm$ 0.00002 \\
Lasso (time)
& 0.00493 $\pm$ 0.00006
& 0.00513 $\pm$ 0.00001
& 0.00588 $\pm$ 0.00004
& 0.00666 $\pm$ 0.00008 \\
\bottomrule
\end{tabular}
\end{table}

\begin{table}[htbp]
\centering
\caption{Ocean XCO$_2$ out-of-time NRMSE. NRMSE is reported for the 2020 in-period test set and the 2021--2023 out-of-time test sets. Entries are means $\pm$ standard deviations over 10 runs.}
\label{tab:OOT_ocean_xco2}
\renewcommand{\arraystretch}{1.2}
\setlength{\tabcolsep}{6pt}

\begin{tabular}{lcccc}
\toprule
\textbf{Model} & \textbf{2020} & \textbf{2021} & \textbf{2022} & \textbf{2023} \\
\midrule
k-NN
& 0.00426 $\pm$ 0.00004
& 0.00540 $\pm$ 0.00001
& 0.00624 $\pm$ 0.00002
& 0.00859 $\pm$ 0.00006 \\
k-NN (time)
& 0.00451 $\pm$ 0.00004
& 0.00796 $\pm$ 0.00003
& 0.01084 $\pm$ 0.00009
& 0.01495 $\pm$ 0.00018 \\
XGBoost
& 0.00341 $\pm$ 0.00005
& 0.00464 $\pm$ 0.00002
& 0.00627 $\pm$ 0.00004
& 0.01032 $\pm$ 0.00007 \\
XGBoost (time)
& \textbf{0.00340} $\pm$ 0.00004
& 0.00466 $\pm$ 0.00004
& 0.00632 $\pm$ 0.00008
& 0.01037 $\pm$ 0.00017 \\
NN
& 0.01182 $\pm$ 0.00058
& 0.01809 $\pm$ 0.00141
& 0.02939 $\pm$ 0.00346
& 0.04814 $\pm$ 0.00691 \\
NN (time)
& 0.01076 $\pm$ 0.00209
& 0.02238 $\pm$ 0.00531
& 0.04255 $\pm$ 0.01128
& 0.07509 $\pm$ 0.02027 \\
Lasso
& 0.00384 $\pm$ 0.00004
& 0.00402 $\pm$ 0.00001
& \textbf{0.00395} $\pm$ 0.00001
& \textbf{0.00416} $\pm$ 0.00001 \\
Lasso (time)
& 0.00383 $\pm$ 0.00004
& \textbf{0.00401} $\pm$ 0.00001
& 0.00396 $\pm$ 0.00001
& 0.00420 $\pm$ 0.00003 \\
\bottomrule
\end{tabular}
\end{table}

\begin{table}[htbp]
\centering
\caption{Land XCH$_4$ out-of-time NRMSE. NRMSE is reported for the 2020 in-period test set and the 2021--2023 out-of-time test sets. Entries are means $\pm$ standard deviations over 10 runs.}
\label{tab:OOT_land_xch4}
\footnotesize
\renewcommand{\arraystretch}{1.2}
\setlength{\tabcolsep}{6pt}
\begin{tabular}{lcccc}
\toprule
\textbf{Model} & \textbf{2020} & \textbf{2021} & \textbf{2022} & \textbf{2023} \\
\midrule
k-NN
& 0.00728 $\pm$ 0.00009
& 0.01326 $\pm$ 0.00002
& 0.02065 $\pm$ 0.00003
& 0.02747 $\pm$ 0.00003 \\
k-NN (time)
& \textbf{0.00628} $\pm$ 0.00007
& 0.01144 $\pm$ 0.00002
& 0.01709 $\pm$ 0.00002
& 0.02294 $\pm$ 0.00004 \\
XGBoost
& 0.01005 $\pm$ 0.00010
& 0.01379 $\pm$ 0.00002
& 0.02041 $\pm$ 0.00006
& 0.02709 $\pm$ 0.00010 \\
XGBoost (time)
& 0.01001 $\pm$ 0.00009
& 0.01317 $\pm$ 0.00007
& 0.01958 $\pm$ 0.00009
& 0.02636 $\pm$ 0.00013 \\
NN
& 0.00629 $\pm$ 0.00035
& 0.01119 $\pm$ 0.00086
& 0.01866 $\pm$ 0.00152
& 0.02526 $\pm$ 0.00208 \\
NN (time)
& 0.00645 $\pm$ 0.00062
& 0.00807 $\pm$ 0.00068
& 0.01013 $\pm$ 0.00195
& 0.01204 $\pm$ 0.00295 \\
Lasso
& 0.00688 $\pm$ 0.00009
& 0.01184 $\pm$ 0.00001
& 0.01921 $\pm$ 0.00003
& 0.02576 $\pm$ 0.00004 \\
Lasso (time)
& 0.00673 $\pm$ 0.00007
& \textbf{0.00725} $\pm$ 0.00003
& \textbf{0.00742} $\pm$ 0.00008
& \textbf{0.00755} $\pm$ 0.00010 \\
\bottomrule
\end{tabular}
\end{table}

\begin{table}[htbp]
\centering
\caption{Ocean XCH$_4$ out-of-time NRMSE. NRMSE is reported for the 2020 in-period test set and the 2021--2023 out-of-time test sets. Entries are means $\pm$ standard deviations over 10 runs.}
\label{tab:OOT_ocean_xch4}
\footnotesize
\renewcommand{\arraystretch}{1.2}
\setlength{\tabcolsep}{6pt}
\begin{tabular}{lcccc}
\toprule
\textbf{Model} & \textbf{2020} & \textbf{2021} & \textbf{2022} & \textbf{2023} \\
\midrule
k-NN
& 0.00871 $\pm$ 0.00008
& 0.01071 $\pm$ 0.00001
& 0.01544 $\pm$ 0.00003
& 0.01968 $\pm$ 0.00002 \\
k-NN (time)
& 0.00812 $\pm$ 0.00007
& 0.02166 $\pm$ 0.00006
& 0.03018 $\pm$ 0.00010
& 0.03480 $\pm$ 0.00019 \\
XGBoost
& 0.00839 $\pm$ 0.00017
& 0.01250 $\pm$ 0.00006
& 0.02084 $\pm$ 0.00013
& 0.02778 $\pm$ 0.00041 \\
XGBoost (time)
& 0.00855 $\pm$ 0.00010
& 0.01256 $\pm$ 0.00009
& 0.02105 $\pm$ 0.00026
& 0.02817 $\pm$ 0.00062 \\
NN
& 0.00843 $\pm$ 0.00296
& 0.01315 $\pm$ 0.00211
& 0.02181 $\pm$ 0.00210
& 0.02919 $\pm$ 0.00180 \\
NN (time)
& 0.00724 $\pm$ 0.00303
& 0.00911 $\pm$ 0.00326
& 0.01416 $\pm$ 0.00473
& 0.01748 $\pm$ 0.00563 \\
Lasso
& 0.00521 $\pm$ 0.00004
& 0.00947 $\pm$ 0.00002
& 0.01701 $\pm$ 0.00006
& 0.02413 $\pm$ 0.00009 \\
Lasso (time)
& \textbf{0.00514} $\pm$ 0.00003
& \textbf{0.00656} $\pm$ 0.00007
& \textbf{0.01014} $\pm$ 0.00020
& \textbf{0.01324} $\pm$ 0.00032 \\
\bottomrule
\end{tabular}
\end{table}

\clearpage
\begin{backmatter}

\section*{Abbreviations}

GHG: greenhouse gas;
GOSAT: Greenhouse Gases Observing Satellite;
TANSO-FTS: Thermal and Near-infrared Sensor for Carbon Observation--Fourier Transform Spectrometer;
SWIR: short-wavelength infrared;
L2: Level 2;
NIES: National Institute for Environmental Studies;
CO$_2$: carbon dioxide;
CH$_4$: methane;
XCO$_2$: column-averaged dry-air mole fraction of carbon dioxide;
XCH$_4$: column-averaged dry-air mole fraction of methane;
TCCON: Total Carbon Column Observing Network;
OCO-2: Orbiting Carbon Observatory-2;
TROPOMI: Tropospheric Monitoring Instrument;
CrIS: Cross-track Infrared Sounder;
OSSE: Observing System Simulation Experiment;
NRMSE: normalized root mean squared error;
Lasso: least absolute shrinkage and selection operator;
NN: neural network;
$k$-NN: $k$-nearest neighbours;
XGBoost: Extreme Gradient Boosting.

\section*{Declarations}

\section*{Availability of data and materials}

The GOSAT data used in this study are based on the GOSAT TANSO-FTS
SWIR Level 1B product V230.231 and the GOSAT TANSO-FTS SWIR Level 2
product V03.00. These products are distributed through the GOSAT Data
Archive System operated by the National Institute for Environmental Studies,
Japan: \url{https://data2.gosat.nies.go.jp/index_en.html}. Access to these
data requires user registration and is subject to the GOSAT/GOSAT-2 data
policy and the terms specified by the data provider.

The TCCON data used for external validation are available from the TCCON
Data Archive: \url{https://tccondata.org/}. The individual TCCON datasets
used in this study are cited in the reference list through their corresponding
dataset references and persistent identifiers.

The code used for the numerical experiments is available from the corresponding
author on reasonable request. The processed datasets are derived from GOSAT
and TCCON products and can be shared only to the extent permitted by the
redistribution conditions of the original data providers.


\section*{Competing interests}

The authors declare that they have no competing interests.

\section*{Funding}

This work was supported by the National Institute for Environmental Studies
through a research collaboration with EURECOM. The National Institute for
Environmental Studies provided scientific and technical support through the
contributions of its affiliated co-authors, as described in the Authors'
contributions section. Apart from these author contributions, the funding body
had no institutional role in the study design, data processing, statistical
analysis, interpretation of results, decision to submit the manuscript, or
manuscript preparation.

\section*{Authors' contributions}

NG carried out the numerical experiments and prepared the figures and tables.
MK supervised the statistical analysis. YS and HY contributed expertise on
GOSAT retrievals, TCCON validation, data interpretation, and remote-sensing
aspects of the study. All authors contributed to the study design,
interpretation of the results, and revision of the manuscript. All authors read
and approved the final manuscript.

\section*{Use of generative AI tools}

During manuscript preparation, the authors used ChatGPT only as an auxiliary
tool for language editing, LaTeX and formatting checks, and creation of the
graphical abstract. The study design, data processing, implementation,
numerical experiments, statistical analysis, interpretation of results, and all
final decisions about the manuscript were carried out by the authors. The
authors reviewed and approved all AI-assisted text and graphical material. No
generative AI tool is listed as an author.



\section*{Acknowledgements}

The authors thank the TCCON team and the site principal investigators for
making the TCCON data used in this study available, and the TCCON
contributors for advice on data use and acknowledgement.




%

\bibliographystyle{peps-art} 
\bibliography{bibfile}  

\end{backmatter}
\end{document}